\documentclass{article}
\usepackage{fullpage}
\usepackage{amssymb}
\usepackage{amsmath}
\usepackage{amsthm}
\usepackage{authblk}

\usepackage[round]{natbib}
\usepackage[colorlinks=true, linkcolor=blue]{hyperref}

\usepackage{amsthm}
\ifx\theorem\undefined
\newtheorem{theorem}{Theorem}
\fi
\ifx\proposition\undefined
\newtheorem{proposition}{Proposition}
\fi
\ifx\BlackBox\undefined
\newcommand{\BlackBox}{\rule{1.5ex}{1.5ex}}  
\fi
\ifx\proof\undefined
\newenvironment{proof}{\par\noindent{\em Proof:\ }}{\hfill\BlackBox\\[.0mm]}
\fi

\usepackage{amssymb}
\usepackage{amsmath,amscd}
\usepackage{amsthm}
\usepackage{color}

 \newcommand{\update}[1]{#1}

\DeclareFontFamily{U}{mathx}{\hyphenchar\font45}
\DeclareFontShape{U}{mathx}{m}{n}{
      <5> <6> <7> <8> <9> <10> gen * mathx
      <10.95> mathx10 <12> <14.4> <17.28> <20.74> <24.88> mathx12
      }{}
\DeclareSymbolFont{mathx}{U}{mathx}{m}{n}
\DeclareMathSymbol{\intop}  {1}{mathx}{"B3}

\let\temp\phi
\let\phi\varphi
\let\varphi\temp

\renewcommand{\sec}{\textsection}

\newcommand{\pr}{\mathbb{P}}

\newcommand{\R}{\mathbb{R}}

\newcommand{\E}{\mathbb{E}}

\newcommand{\normalN}{\mathcal{N}}
\renewcommand{\|}{\,|\,}            
\newcommand{\given}{\,|\,}  

\newcommand{\norm}[1]{\Vert#1\Vert}

\newcommand{\grad}{\nabla}

\DeclareMathOperator{\tr}{tr}

\DeclareMathOperator*{\argmin}{arg\,min}

\DeclareMathOperator{\ExpDist}{Exp}

\DeclareMathOperator{\UniformDist}{Unif}
\DeclareMathOperator{\GumbelDist}{Gumbel}

\usepackage{microtype}
\usepackage{graphicx}
\usepackage{booktabs}
\usepackage{algorithm}

\usepackage{enumitem}
\usepackage{pgfplots}
\usepackage{subcaption}
\usepackage{cleveref}

\usepackage{color}
\usepackage{enumitem}

\newcommand{\dagspace}{\mathbb{D}} 
\newcommand{\wadj}{\R^{d\times d}}
\newcommand{\bin}{\{0,1\}^{d\times d}}

\newcommand{\adj}{\mathcal{A}}

\newcommand{\ecp}{\mathsf{ECP}}
\newcommand{\spec}{r}

\newcommand{\Gauss}{\mathsf{Gauss}}
\newcommand{\Exp}{\mathsf{Exp}}
\newcommand{\Gumbel}{\mathsf{Gumbel}}

\newcommand{\FGS}{\text{FGS}}
\newcommand{\fgsest}{B_{\FGS}}

\newcommand{\dat}{\mathbf{X}}

\newcommand{\method}{\text{NOTEARS}}
\newcommand{\methodreg}{\text{NOTEARS-}\ell_{1}}
\newcommand{\true}{W}
\newcommand{\truecol}{w}
\newcommand{\trueadj}{B}
\newcommand{\est}{\widehat{W}}
\newcommand{\adjest}{\widehat{B}}
\newcommand{\estecp}{\widetilde{W}_{\ecp}}
\newcommand{\estgob}{W_{\gr}}

\newcommand{\loss}{\ell}
\newcommand{\scoregr}{Q}
\newcommand{\score}{F}
\newcommand{\gr}{\mathsf{G}}
\newcommand{\ver}{\mathsf{V}}
\newcommand{\edg}{\mathsf{E}}
\DeclareMathOperator{\pa}{pa}

\DeclareMathOperator{\vect}{vec}

\title{\Large{DAGs with NO TEARS: Continuous Optimization for Structure Learning}}
\author[]{Xun Zheng}
\author[]{Bryon Aragam}
\author[]{Pradeep Ravikumar}
\author[]{Eric P. Xing}
\affil[]{\emph{Carnegie Mellon University}}

\begin{document}
\maketitle

{\let\thefootnote\relax\footnote{Contact: \texttt{xzheng1@andrew.cmu.edu, naragam@cs.cmu.edu, pradeepr@cs.cmu.edu, epxing@cs.cmu.edu}}}

\begin{abstract}
Estimating the structure of directed acyclic graphs (DAGs, also known as {Bayesian networks}) is a challenging problem since the search space of DAGs is combinatorial and scales superexponentially with the number of nodes. Existing approaches rely on various local heuristics for enforcing the acyclicity constraint. In this paper, we introduce a fundamentally different strategy: We formulate the structure learning problem as a purely \emph{continuous} optimization problem over real matrices that avoids this combinatorial constraint entirely. 
This is achieved by a novel characterization of acyclicity that is not only smooth but also exact.  The resulting problem can be efficiently solved by standard numerical algorithms, which also makes implementation effortless. The proposed method outperforms existing ones, without imposing any structural assumptions on the graph such as bounded treewidth or in-degree.
Code implementing the proposed algorithm is open-source and publicly available at \texttt{\url{https://github.com/xunzheng/notears}}.
\end{abstract}

\section{Introduction}
\label{sec:intro}

Learning directed acyclic graphs (DAGs) from data is an NP-hard problem \citep{chickering1996,chickering2004}, owing mainly to the combinatorial acyclicity constraint that is difficult to enforce efficiently. At the same time, DAGs are popular models in practice, with applications in biology \citep{sachs2005}, genetics \citep{zhang2013}, machine learning \citep{koller2009}, and causal inference \citep{spirtes2000}. For this reason, the development of new methods for learning DAGs remains a central challenge in machine learning and statistics. 

In this paper, we propose a new approach for score-based learning of DAGs by converting the traditional \emph{combinatorial} optimization problem (left) into a \emph{continuous} program (right):
\begin{align}
\label{eq:main:idea}
\begin{aligned}
\min_{W\in\wadj} & \ \ \score(W) \\
\text{subject to} & \ \ \gr(W) \in \mathsf{DAGs}
\end{aligned}
\quad \iff \quad 
\begin{aligned}
\min_{W\in\wadj} & \ \ \score(W) \\
\text{subject to} & \ \ h(W) = 0,
\end{aligned}
\end{align}
\noindent
where $\gr(W)$ is the $d$-node graph induced by the weighted adjacency matrix $W$,
$\score: \mathbb{R}^{d \times d} \to \mathbb{R} $ is a {score function} (see Section~\ref{sec:background:score} for details), 
and our key technical device $h: \mathbb{R}^{d \times d} \to \mathbb{R} $ is a smooth function over real matrices, \update{whose level set at zero exactly characterizes acyclic graphs.}
Although the two problems are equivalent, the continuous program on the right eliminates the need for specialized algorithms that are tailored to search over the combinatorial space of DAGs. 
Instead, we are able to leverage standard numerical algorithms for constrained problems, which makes implementation particularly easy, not requiring any knowledge about graphical models.
This is similar in spirit to the situation for undirected graphical models, in which the formulation of a \update{continuous} log-det program \citep{banerjee2008} sparked a series of remarkable advances in structure learning for undirected graphs (Section~\ref{sec:background:prev}).
\update{Unlike undirected models, which can be reduced to a convex program, however, the program \eqref{eq:main:idea} is \emph{nonconvex}. Nonetheless, as we will show, even na\"ive solutions to this program yield state-of-the-art results for learning DAGs.}

\paragraph{Contributions.} The main thrust of this work is to re-formulate score-based learning of DAGs so that 
standard smooth optimization
schemes such as L-BFGS~\citep{nocedal2006numerical} can be leveraged. 
To accomplish this, we make the following specific contributions:
\begin{itemize}
\item We explicitly construct a smooth function over $\wadj$ with computable derivatives that encodes the acyclicity constraint. This allows us to replace the combinatorial constraint $\gr\in\dagspace$ in \eqref{eq:main:opt} with a smooth equality constraint.
\item We develop an 
equality-constrained
program for simultaneously estimating the structure and parameters of a sparse DAG from possibly high-dimensional data, and show how standard numerical solvers can be used to find stationary points.
\item We demonstrate the effectiveness of the resulting method in empirical evaluations against existing state-of-the-arts.
See Figure~\ref{fig:compare:heatmap:er2} for a quick illustration and Section~\ref{sec:exp} for details.
\item We compare our ouput to the exact global minimizer \citep{cussens2012}, and show that our method attains scores that are comparable to the globally optimal score in practice, although our methods are only guaranteed to find stationary points.
\end{itemize}

Most interestingly, our approach is very simple and can be implemented in about 50 lines of Python code. As a result of its simplicity and effortlessness in its implementation, we call the resulting method $\method$: \emph{Non-combinatorial Optimization via Trace Exponential and Augmented lagRangian for Structure learning.} 
\update{The implementation is publicly available at \texttt{\url{https://github.com/xunzheng/notears}}.}

\begin{figure}[t]
\centering
\begin{subfigure}[t]{0.138\textwidth}
\centering
\includegraphics[width=0.99\textwidth]{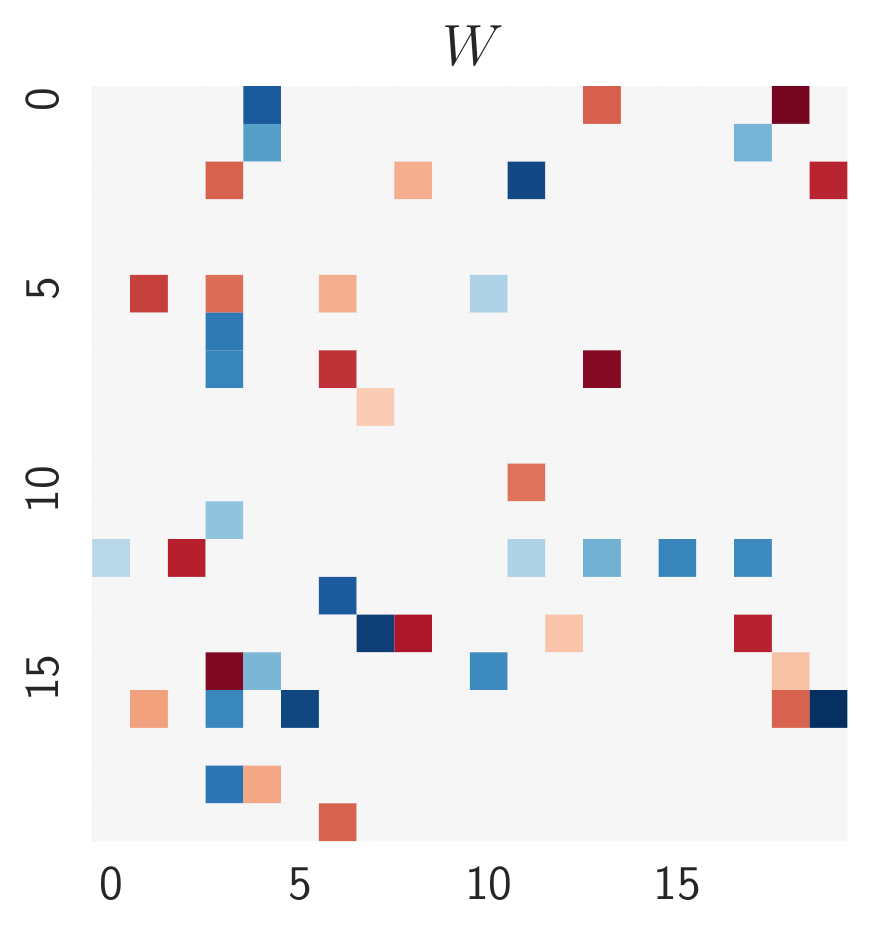}
\caption{true graph}
\end{subfigure}%
~
\begin{subfigure}[t]{0.43\textwidth}
\centering
\includegraphics[width=0.99\textwidth]{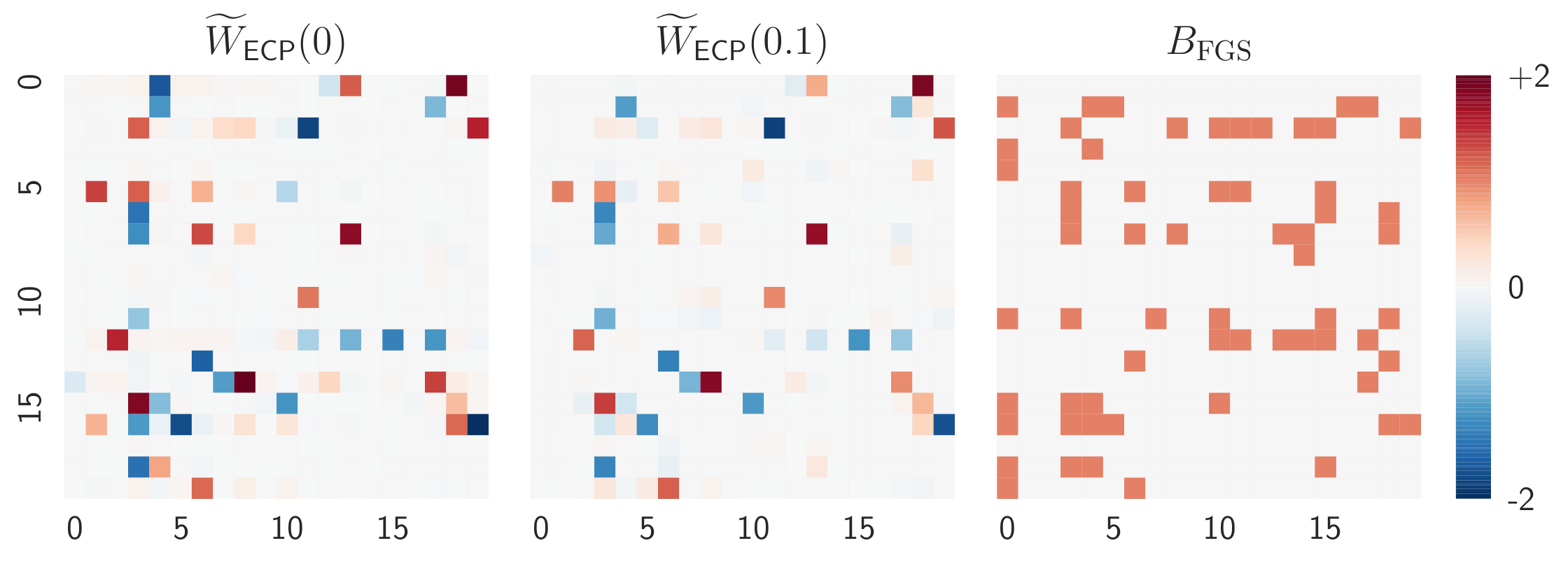}
\caption{estimate with $n=1000$}
\end{subfigure}%
\begin{subfigure}[t]{0.43\textwidth}
\centering
\includegraphics[width=0.99\textwidth]{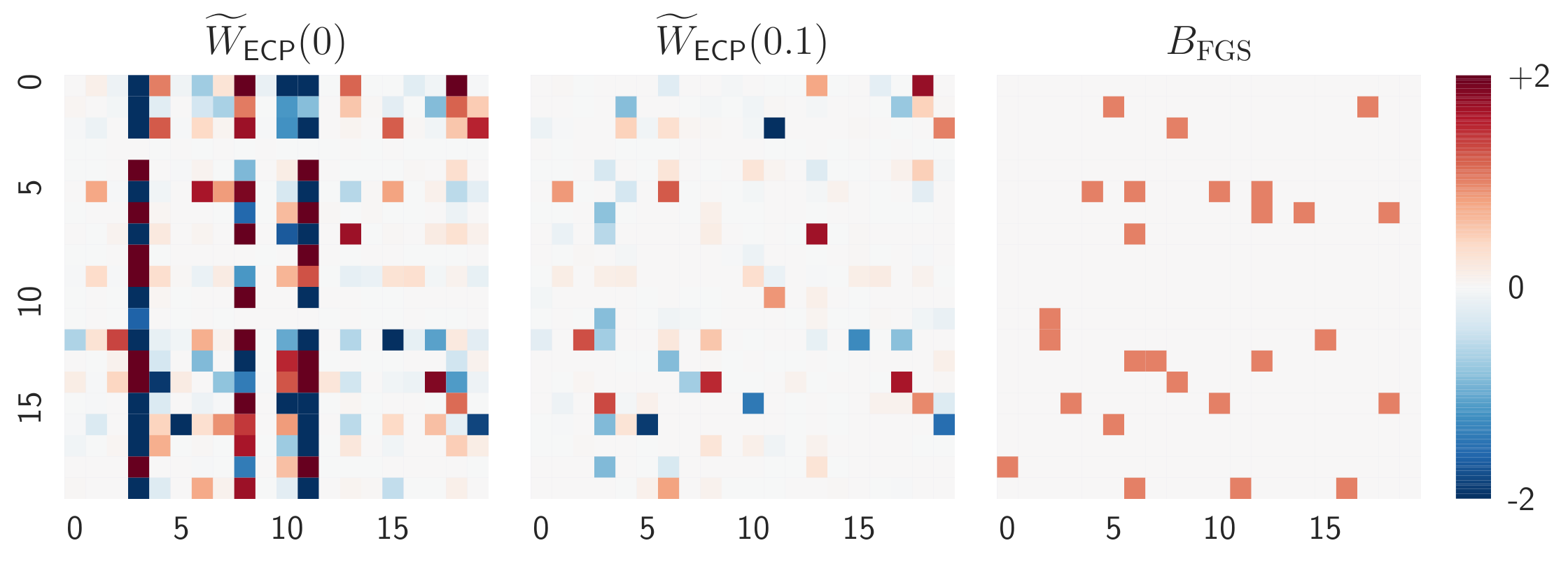}
\caption{estimate with $ n=20 $}
\end{subfigure}
\caption{Visual comparison of the learned weighted adjacency matrix on a 20-node graph with $ n=1000 $ (large samples) and $ n=20 $ (insufficient samples): $ \estecp(\lambda) $ is the proposed $ \method $ algorithm with $ \ell_1 $-regularization $ \lambda $, and $ \fgsest $ is the binary estimate of the baseline~\citep{ramsey2016}. 
The proposed algorithms perform well on large samples, and remains accurate on small $ n $ with $ \ell_1 $ regularization.}
\label{fig:compare:heatmap:er2}
\end{figure}

\section{Background}
\label{sec:background}

The basic DAG learning problem is formulated as follows: Let $\dat\in\R^{n\times d}$ be a data matrix consisting of $n$ i.i.d. observations of the random vector $X=(X_{1},\ldots,X_{d})$ and let $\dagspace$ denote the (discrete) space of DAGs $\gr=(\ver,\edg)$ on $d$ nodes. Given $\dat$, we seek to learn a DAG $\gr\in\dagspace$ (also called a \emph{Bayesian network}) for the joint distribution $\pr(X)$ \citep{spirtes2000,koller2009}. We model $X$ via a structural equation model (SEM) defined by a weighted adjacency matrix $W\in\wadj$. Thus, instead of operating on the discrete space $\dagspace$, we will operate on $\wadj$, the continuous space of $d\times d$ real matrices. 

\subsection{Score functions and SEM}
\label{sec:background:score}

Any $W\in\wadj$ defines a graph on $d$ nodes in the following way: Let $\adj(W)\in\bin$ be the binary matrix such that $[\adj(W)]_{ij}=1\iff w_{ij}\ne 0$ and zero otherwise; then $\adj(W)$ defines the adjacency matrix of a directed graph $\gr(W)$. In a slight abuse of notation, we will thus treat $W$ as if it were a (weighted) graph. In addition to the graph $\gr(W)$, $W=[\,w_{1}\|\cdots\|w_{d}\,]$ defines a linear SEM by $X_{j}=w_{j}^{T}X+z_{j}$, where $X=(X_{1},\ldots,X_{d})$ is a random vector and $z=(z_{1},\ldots,z_{d})$ is a random noise vector. We do \emph{not} assume that $z$ is Gaussian. More generally, we can model $X_{j}$ via a generalized linear model (GLM) $\E(X_{j}\given X_{\pa(X_{j})})=f(w_{j}^{T}X)$. For example, if $X_{j}\in\{0,1\}$, we can model the conditional distribution of $X_{j}$ given its parents via logistic regression. 

In this paper, we focus on linear SEM and the least-squares (LS) loss
$\loss(W;\dat)=\frac{1}{2n}\norm{\dat - \dat W}_{F}^{2}$,
although everything in the sequel applies to any smooth loss function $\loss$ defined over $\wadj$. The statistical properties of the LS loss in scoring DAGs have been extensively studied: The minimizer of the LS loss provably recovers a true DAG with high probability on finite-samples and in high-dimensions ($d\gg n$), and hence is consistent for both Gaussian SEM \citep{geer2013,aragam2016} and non-Gaussian SEM \citep{loh2014causal}.\footnote{Due to nonconvexity, there may be more than one minimizer: These and other technical issues such as parameter identifiability are addressed in detail in the cited references.} Note also that these results imply that the faithfulness assumption is not required in this set-up. Given this extensive previous work on statistical issues, our focus in this paper is entirely on the computational problem of finding an SEM that minimizes the LS loss.

This translation between graphs and SEM is central to our approach. Since we are interested in learning a \emph{sparse} DAG, we add $\ell_{1}$-regularization $ \norm{W}_{1} = \norm{\vect (W)}_1 $ resulting in the regularized score function
\begin{align}
\score(W) 
= \loss(W;\dat) + \lambda\norm{W}_{1}
= \frac{1}{2n}\norm{\dat - \dat W}_{F}^{2} + \lambda\norm{W}_{1}.
\end{align}
Thus we seek to solve
\begin{align}
\label{eq:exact:program}
\begin{aligned}
\min_{W\in\wadj} & \quad \score(W) \\
\text{subject to} & \quad  \gr(W)\in\dagspace.
\end{aligned} 
\end{align}
\noindent
Unfortunately, although $\score(W)$ is continuous, the DAG constraint $\gr(W)\in\dagspace$ remains a challenge to enforce. In Section~\ref{sec:acyclicity}, we show how this discrete constraint  can be replaced by a smooth equality constraint.

\subsection{Previous work}
\label{sec:background:prev}

Traditionally, score-based learning seeks to optimize a \emph{discrete score} $\scoregr:\dagspace\to\R$ over the set of DAGs $\dagspace$;
note that this is distinct from our score $\score(W)$ whose domain is $\wadj$ instead of $\dagspace$. 
This can be written as the following 
combinatorial optimization problem:
\begin{align}
\label{eq:main:opt}
\begin{aligned}
\min_{\gr} \quad  & \quad \scoregr(\gr) \\
\text{subject to} & \quad \gr\in\dagspace
\end{aligned}
\end{align}
\noindent
Popular score functions include BDe(u) \citep{heckerman1995}, BGe \citep{kuipers2014}, BIC \citep{chickering1997}, and MDL \citep{bouckaert1993}. Unfortunately, \eqref{eq:main:opt} is NP-hard to solve \cite{chickering1996,chickering2004} owing mainly to the nonconvex, combinatorial nature of the optimization problem. 
This is the main drawback of existing approaches for solving \eqref{eq:main:opt}:
The acyclicity constraint is a combinatorial constraint with the number of acyclic structures increasing superexponentially in $d$ \citep{robinson1977}. 
Notwithstanding, there are algorithms for solving \eqref{eq:main:opt} to global optimality for small problems \citep{ott2003,singh2005,silander2012,xiang2013,cussens2012,cussens2017}. There is also a wide literature on approximate algorithms based on order search \citep{teyssier2012,schmidt2007,scanagatta2015,scanagatta2016}, greedy search \citep{heckerman1995,chickering2003,ramsey2016}, and coordinate descent \citep{fu2013,aragam2015,gu2018}. \update{By searching over the space of topological orderings, the former order-based methods trade-off the difficult problem of enforcing acyclicity with a search over $d!$ orderings, whereas the latter methods enforce acyclicity one edge at a time, explicitly checking for acyclicity violations each time an edge is added.}
Other approaches that avoid optimizing \eqref{eq:main:opt} directly include constraint-based methods \citep{spirtes1991,spirtes2000}, hybrid methods \citep{tsamardinos2006,gamez2011}, and Bayesian methods \citep{ellis2008,zhou2011,niinimaki2016}. 

The intractable form of the program \eqref{eq:main:opt} has led to a host of heuristic methods, often borrowing tools from the optimization literature, but always resorting to clever heuristics to accelerate algorithms. 
Here we briefly discuss some of the pros and cons of existing methods. While not all methods suffer from \emph{all} of the problems highlighted below, we are not aware of any methods that simultaneously avoid all of them.

\paragraph{Exact vs. approximate.}
Broadly speaking, there are two camps: \emph{Approximate} algorithms and \emph{exact} algorithms, the latter of which are guaranteed to return a globally optimal solution. Exact algorithms form an intriguing class of methods, but as they are based around an NP-hard combinatorial optimization problem, these methods remain computationally intractable in general. For example, recent state-of-the-art work \citep{cussens2017,chen2016bn} only scale to problems with a few dozen nodes \citep{van-beek2015}.\footnote{\citet{cussens2012} reports experiments with $d>60$ under a constraint on the maximum parent size.} Older methods based on dynamic programming methods \citep{ott2003,singh2005,silander2012,xiang2013,loh2014causal} also scale to roughly a few dozen nodes. By contrast, state-of-the-art approximate methods can scale to thousands of nodes \citep{ramsey2016,aragam2015,scanagatta2015,scanagatta2016}. 

\paragraph{Local vs. global search.}
Arguably the most popular approaches to optimizing \eqref{eq:main:opt} involve \emph{local} search, wherein edges and parent sets are added sequentially, one node at a time. This is efficient as long as each node has only a few parents, but as the number of possible parents grows, local search rapidly becomes intractable. Furthermore, such strategies typically rely on severe structural assumptions such as bounded in-degree, bounded treewidth, or edge constraints.
Since real-world networks often exhibit scale-free and small-world topologies \citep{watts1998,barabasi1999} with highly connected hub nodes, these kinds of structural assumptions are not only difficult to satisfy, but impossible to check. 
We note here promising work towards relaxing this assumption for discrete data \citep{scanagatta2015}. By contrast, our method uses \emph{global} search wherein the entire matrix $W$ is updated in each step.

\paragraph{Model assumptions.}
The literature on DAG learning tends to be split between methods that operate on discrete data vs. methods that operate on continuous data. When viewed from the lens of \eqref{eq:exact:program}, the reasons for this are not clear since both discrete and continuous data can be considered as special cases of the general score-based learning framework. Nonetheless, many (but not all) of the methods cited already only work under very specific assumptions on the data, the most common of which are categorical (discrete) and Gaussian (continuous). Since \eqref{eq:exact:program} is agnostic to the form of the data and loss function, there is significant interest in finding general methods that are not tied to specific model assumptions.

\paragraph{Conceptual clarity.}
Finally, on a higher level, a significant drawback of existing methods is their conceptual complexity: They are not straightforward to implement, require deep knowledge of concepts from the graphical modeling literature, and accelerating them involves many clever tricks. By contrast, the method we propose in this paper is conceptually very simple, requires no background on graphical models, and can be implemented in just a few lines of code using existing black-box solvers.

\subsection{Comparison}
 
It is instructive to compare existing methods for learning DAGs against other methods in the machine learning literature. We focus here on two popular models: Undirected graphical models and deep neural networks. Undirected graphical models, also known as Markov networks, is recognized as a convex problem \citep{yuan2007,banerjee2008} nowadays, and hence can be solved using black-box convex optimizers such as CVX \citep{cvx}. 
However, one should not forget score-based methods based on discrete scores similar to \eqref{eq:main:opt} proliferated in the early days for learning undirected graphs \citep[e.g.][\sec20.7]{koller2009}.
More recently, extremely efficient algorithms have been developed for this problem using coordinate descent \citep{friedman2008} and Newton methods \citep{hsieh2014quic,schmidt2009}. 
As another example, deep neural networks are often learned using various descendants of stochastic gradient descent (SGD) \citep{bousquet2008,kingma2014,bottou2016}, although recent work has proposed other techniques such as ADMM \citep{taylor2016} and Gauss-Newton \citep{botev2017}. One of the keys to the success of both of these models---and many other models in machine learning---was having a closed-form, tractable program for which existing techniques from the extensive optimization literature could be applied. In both cases, the application of principled optimization techniques led to significant breakthroughs. For undirected graphical models the major technical tool was convex optimization, and for deep networks the major technical tool was SGD. 

Unfortunately, the general problem of DAG learning has not benefited in this way, and one of our main goals in the current work is to formulate score-based learning similarly as a closed-form, continuous program. Arguably, the challenges with existing approaches stem from the intractable form of the program \eqref{eq:main:opt}. 
One of our main goals in the current work is to formulate score-based learning via a similar closed-form, continuous program. The key device in accomplishing this is a smooth characterization of acyclicity that will be introduced in the next section.

\section{A new characterization of acyclicity}
\label{sec:acyclicity}

In order to make \eqref{eq:exact:program} amenable to black-box optimization, we propose to replace the combinatorial acyclicity constraint $\gr(W)\in\dagspace$ in \eqref{eq:exact:program} with a single smooth equality constraint $h(W)=0$. 
Ideally, we would like a function $h: \wadj \to \R$ that satisfies the following desiderata: 
\begin{enumerate}[label=(\alph*), itemsep=0pt]
\item $h(W)=0$ if and only if $W$ is acyclic (i.e. $\gr(W)\in\dagspace$);
\item The values of $h$ quantify the ``DAG-ness'' of the graph;
\item $h$ is smooth;
\item $h$ and its derivatives are easy to compute.
\end{enumerate}

\update{
Property (b) is useful in practice for diagnostics.
By ``DAG-ness'', we mean some quantification of how severe violations from acyclicity become as $W$ moves further from $\dagspace$.
Although there are many ways to satisfy (b) by measuring some notion of ``distance'' to $\dagspace$, typical approaches would violate (c) and (d).
For example, 
$h$ might be the minimum $\ell_2$ distance to $\dagspace$
or it might be the sum of edge weights along all cyclic paths of $W$,
however, these are either non-smooth (violating (c)) or  hard to compute (violating (d)).
If a function that satisfies desiderata (a)-(d) exists, we can hope to apply existing machinery for constrained optimization such as Lagrange multipliers. Consequently, the DAG learning problem becomes equivalent to solving a numerical optimization problem, which is agnostic about the graph structure. 
}

We proceed in two steps: First, we consider the simpler case of binary adjacency matrices $B \in \bin$ (Section~\ref{app:acyclicity:bin}).
Note that since $\bin$ is a discrete space, we cannot take gradients or do continuous optimization. For this we need the second step, in which we relax the function we originally define on binary matrices to real matrices (Section~\ref{app:acyclicity:wadj}).

\subsection{Special case: Binary adjacency matrices}
\label{app:acyclicity:bin}
 
When does a matrix $B\in\bin$ correspond to an acyclic graph? 
Recall the \emph{spectral radius} $\spec(B)$ of a matrix $B$ is the largest absolute eigenvalue of $B$. One simple characterization of acyclicity is the following:

\begin{proposition}[Infinite series]
\label{prop:traceinv}
Suppose $B\in\bin$ and $\spec(B)<1$. Then $B$ is a DAG if and only if
\begin{align}
\tr (I - B)^{-1} = d. 
\end{align}
\end{proposition}
\begin{proof}
It essentially boils down to the fact that $\tr B^{k}$ counts the number of length-$k$ closed walks in a directed graph. Clearly an acyclic graph will have $\tr B^{k}=0$ for all $k = 1, \dotsc, \infty $. 
In other words, $B$ has no cycles if and only if $f(B)=\sum_{k=1}^{\infty}\sum_{i=1}^{d}(B^{k})_{ii}=0$,  then 
\begin{align*}
\tr(I-B)^{-1}
= \tr\sum_{k=0}^{\infty}B^{k} 
= \tr I + \sum_{k=1}^{\infty}\tr B^{k} 
= d + \sum_{k=1}^{\infty}\sum_{i=1}^{d} (B^{k})_{ii} 
= d + f(B).
\end{align*}
The desired result follows.
\end{proof}

Unfortunately, the condition that $\spec(B)<1$ is strong: although it is automatically satisfied when $ B $ is a DAG, it is generally not true otherwise, and furthermore the projection is nontrivial. 
Alternatively, instead of the infinite series, one could consider the characterization based on \emph{finite} series $ \sum_{k=1}^{d} \tr B^{k} = 0 $, which does not require $ \spec(B) < 1 $. 
However, this is impractical for numerical reasons: The entries of $B^{k}$ can easily exceed machine precision for even small values of $ d $, which makes both function and gradient evaluations highly unstable. Therefore it remains to find a characterization that not only holds for all possible $ B $, but also has numerical stability.  
Luckily, such function exists.

\begin{proposition}[Matrix exponential]
\label{prop:matrixexp}
A binary matrix $B\in\bin$ is a DAG if and only if
\begin{align}
\label{eq:prop:matrixexp}
\tr e^{B} = d.
\end{align}
\end{proposition}
\begin{proof}
Similar to Proposition~\ref{prop:traceinv} by noting that $B$ has no cycles if and only if $(B^{k})_{ii}=0$ for all $k\ge 1$ and all $i$, which is true if and only if $\sum_{k=1}^{\infty}\sum_{i=1}^{d}(B^{k})_{ii}/k!=\tr e^{B}-d=0$.
\end{proof}

It is worth pointing out that matrix exponential is well-defined for all square matrices. 
In addition to everywhere convergence, this characterization has an added bonus: As the number of edges in $B$ increases along with the number of nodes $d$, the number of possible closed walks grows rapidly, so the trace characterization $\tr(I-B)^{-1}$ rapidly becomes ill-conditioned and difficult to manage. By re-weighting the number of length-$k$ closed walks by $k!$, this becomes much easier to manage.
While this is a useful characterization, it does not satisfy all of our desiderata since---being defined over a discrete space---it is not a smooth function. The final step is to extend Proposition~\ref{prop:matrixexp} to all of $\R^{d\times d}$.

\subsection{The general case: Weighted adjacency matrices}
\label{app:acyclicity:wadj}

Unfortunately, the characterization \eqref{eq:prop:matrixexp} fails if we replace $B$ with an arbitrary weighted matrix $W$. However, we can replace $B$ with any \emph{nonnegative} weighted matrix, and the same argument use to prove Proposition~\ref{prop:matrixexp} shows that \eqref{eq:prop:matrixexp} will still characterize acyclicity. Thus, to extend this to matrices with both positive and negative values, we can simply use the Hadamard product $W\circ W$, which leads to our main result.

\begin{theorem}
\label{thm:ideal}
A matrix $W\in\wadj$ is a DAG if and only if 
\begin{align}
\label{eq:def:h}
h(W) = \tr \big( e^{W \circ W} \big) - d = 0,
\end{align}
where $ \circ $ is the Hadamard product and $ e^A $ is the matrix exponential of $ A $. Moreover, $ h(W) $ has a simple gradient 
\begin{align}
\label{eq:def:gradh}
\grad h(W) = \big( e^{W \circ W} \big)^T \circ 2W,
\end{align}
and satisfies all of the desiderata (a)-(d).
\end{theorem}

The proof of \eqref{eq:def:h} is similar to \eqref{eq:prop:matrixexp}, and desiderata (c)-(d) follow from \eqref{eq:def:gradh}. To see why desiderata (b) holds, note that the proof of Proposition~\ref{prop:traceinv} shows that the power series $\tr(B+B^{2}+\cdots)$ simply counts the number of closed walks in $B$, and the matrix exponential simply re-weights these counts. Replacing $B$ with $W\circ W$ amounts to counting \emph{weighted} closed walks, where the weight of each edge is $w_{ij}^{2}$. Thus, larger $h(W)>h(W')$ means either (a) $W$ has more cycles than $W'$ or (b) The cycles in $W$ are more heavily weighted than in $W'$.

Moreover, notice that $ h(W) \ge 0 $ for all $ W $ since each term in the series is nonnegative. 
This gives another interesting perspective of the space of DAGs as the set of global minima of $ h(W) $.
However, due to the nonconvexity, this is not equivalent to the first order stationary condition $ \nabla h(W) = 0 $.

A key conclusion from Theorem~\ref{thm:ideal} is that $h$ and its gradient only involve evaluating the matrix exponential, which is a well-studied function in numerical analysis, and whose $ O(d^3) $ algorithm~\citep{almohy2009new} is readily available in many scientific computing libraries. 
\update{
Although the connection between trace of matrix power and number of cycles in the graph is well-known~\cite{harary1971number},}
to the best of our knowledge, this characterization of acyclicity has not appeared in the 
\update{DAG learning}
literature previously.
We defer the discussion of other possible characterizations in the appendix. 
In the next section, we apply Theorem~\ref{thm:ideal} to solve the program \eqref{eq:exact:program} to stationarity by treating it as an equality constrained program.

\section{Optimization}
\label{sec:opt}

Theorem~\ref{thm:ideal} establishes a smooth, algebraic characterization of acyclicity that is also computable.
As a consequence, the following equality-constrained program $(\ecp)$ is equivalent to \eqref{eq:exact:program}:
\begin{align}
\label{eq:ideal:program}
(\ecp) \qquad
\begin{aligned}
\min_{W\in\wadj} & \quad \score(W) \phantom{(\ecp) \qquad} \\
\text{subject to} & \quad h(W)=0.
 \end{aligned}
\end{align}
The main advantage of $(\ecp)$ compared to both \eqref{eq:exact:program} and \eqref{eq:main:opt} is its amenability to classical techniques from the mathematical optimization literature. 
Nonetheless, since $\{W: h(W)=0\}$ is a nonconvex constraint, \eqref{eq:ideal:program} is a nonconvex program, hence we still inherit the difficulties associated with nonconvex optimization. In particular, we will be content to find stationary points of \eqref{eq:ideal:program}; 
{in Section~\ref{sec:exp:global} we compare our results to the global minimizer and show that the stationary points found by our method are close to global minima in practice.}

In the follows, we outline the algorithm for solving \eqref{eq:ideal:program}. 
It consists of three steps: (i) converting the \emph{constrained} problem into a sequence of \emph{unconstrained} subproblems, (ii) optimizing the unconstrained subproblems, and (iii) thresholding.
The full algorithm is outlined in Algorithm~\ref{alg:notears}. 

\begin{algorithm}[t]
\caption{$\method$ algorithm}
\label{alg:notears}
\begin{enumerate}
\item Input: 
Initial guess $ (W_{0}, \alpha_0) $, 
progress rate $ c \in (0,1) $, 
tolerance $ \epsilon > 0 $, 
threshold $\omega>0$.
\item For $ t = 0, 1, 2, \dotsc$:
    \begin{enumerate}
    \item Solve primal $ W_{t+1} \gets \argmin_W L^\rho (W, \alpha_t) $ with $ \rho $ such that $ h(W_{t+1}) < c h(W_{t}) $.
    \item Dual ascent $ \alpha_{t+1} \gets \alpha_{t} + \rho h(W_{t+1}) $.
    \item If $  h(W_{t+1})  < \epsilon  $, set $\estecp=W_{t+1}$ and break.
    \end{enumerate}
\item Return the thresholded matrix $\est:=\estecp \circ 1 (|\estecp| > \omega) $.
\end{enumerate}
\end{algorithm}

\subsection{Solving the ECP with augmented Lagrangian}
\label{sec:opt:ecp}

We will use the augmented Lagrangian method \citep[e.g.][]{nemirovski1999optimization} to solve $(\ecp)$, which solves the original problem augmented by a quadratic penalty:
\begin{align}
\label{eqn:augmented-lagrangian}
\begin{aligned}
\min_{W\in\wadj} & \quad \score(W) + \frac{\rho}{2} |h(W)|^2 \\
\text{subject to} & \quad h(W) = 0
\end{aligned}
\end{align}
with a penalty parameter $ \rho > 0 $. 
A nice property of the augmented Lagrangian method is that it approximates well the solution of a \emph{constrained} problem by the solution of  \emph{unconstrained} problems \emph{without} increasing the penalty parameter $ \rho $ to infinity~\citep{nemirovski1999optimization}. 
The algorithm is essentially a dual ascent method for \eqref{eqn:augmented-lagrangian}. 
To begin with, 
the dual function with Lagrange multiplier $ \alpha $ is given by
\begin{gather}
\label{eqn:min-aug-primal}
D(\alpha) = \min_{W \in \R^{d \times d}} L^\rho (W, \alpha), \\ 
\text{where}  \quad  L^\rho (W, \alpha)  = \score(W)  + \frac{\rho}{2} |h(W)|^2 + \alpha h(W)
\end{gather}
is the augmented Lagrangian. 
The goal is to find a local solution to the dual problem
\begin{align}\label{eqn:dual}
\max_{\alpha \in \R} \ \ D(\alpha).
\end{align}
Let $ W_\alpha^\star $ be the local minimizer of the Lagrangian \eqref{eqn:min-aug-primal} at $ \alpha $, i.e. $ D(\alpha) = L^\rho(W_\alpha^\star, \alpha) $.
Since the dual objective $ D(\alpha) $ is linear in $ \alpha $, the derivative is simply given by $\grad D(\alpha) = h(W_\alpha^\star)$.
Therefore one can perform dual gradient ascent to optimize \eqref{eqn:dual}:
\begin{align}\label{eqn:dual-ascent}
\alpha \gets \alpha + \rho h( W_\alpha^\star ),
\end{align}
where the choice of step size $ \rho $ comes with the following convergence rate:
\begin{proposition}[Corollary 11.2.1, \citealp{nemirovski1999optimization}]
For $ \rho $ large enough and the starting point $ \alpha_0 $ near the solution $ \alpha^\star $, the update \eqref{eqn:dual-ascent} converges to $ \alpha^\star $ linearly. 
\end{proposition}

In our experiments, typically fewer than 10 steps of the augmented Lagrangian scheme are required.

\subsection{Solving the unconstrained subproblem}
\label{sec:opt:inner}

\newcommand{\dvec}{\mathbf{\boldsymbol{d}}}
\newcommand{\wvec}{\mathbf{\boldsymbol{w}}}
\newcommand{\gvec}{\mathbf{\boldsymbol{g}}}
\newcommand{\svec}{\mathbf{\boldsymbol{s}}}
\newcommand{\yvec}{\mathbf{\boldsymbol{y}}}

The augmented Lagrangian converts a \emph{constrained} problem \eqref{eqn:augmented-lagrangian} into a sequence of \emph{unconstrained} problems~\eqref{eqn:min-aug-primal}. We now discuss how to solve these subproblems efficiently. Let $ \wvec = \vect (W) \in \R^p $, with $ p = d^2 $. 
The unconstrained subproblem~\eqref{eqn:min-aug-primal} can be considered as a typical minimization problem over real vectors:
\begin{gather}
\min_{\wvec \in \R^{p}} f(\wvec) + \lambda\norm{\wvec}_{1}, \\
\text{where} \quad f(\wvec) = \loss(W;\dat)  + \frac{\rho}{2} |h(W)|^2 + \alpha h(W)
\end{gather}
is the smooth part of the objective. 
Our goal is to solve the above problem to high accuracy so that $ h(W) $ can be sufficiently suppressed.

In the special case of $ \lambda = 0 $, the nonsmooth term vanishes and the problem simply becomes an unconstrained smooth minimization, for which 
a number of efficient numerical algorithms are available, for instance the L-BFGS~\citep{byrd1995limited}. 
To handle the nonconvexity, a slight modification~\citep[Procedure 18.2]{nocedal2006numerical} needs to be applied.

When $ \lambda > 0 $, the problem becomes composite minimization, which can also be efficiently solved by the proximal quasi-Newton (PQN) method~\citep{zhong2014}. 
At each step $ k $, the key idea is to find the descent direction through a quadratic approximation of the smooth term:
\begin{align}\label{eqn:descent-direction-subproblem}
\dvec_k = \argmin_{\dvec \in \R^{p}} \ \gvec_k^T \dvec + \frac{1}{2} \dvec^T B_k \dvec + \lambda \norm{\wvec_k + \dvec}_1,
\end{align}
where $ \gvec_k  $ is the gradient of $ f(\wvec) $ and $ B_k $  is the L-BFGS approximation of the Hessian.
Note that for each coordinate $ j $, problem \eqref{eqn:descent-direction-subproblem} has a closed form update $ \dvec \gets \dvec + z^\star e_j $ given by
\begin{align}\label{eqn:coordinate-update}
z^\star
 = \argmin_z \ \frac{1}{2} \underbrace{B_{jj}}_{a} z^2 + (\underbrace{\gvec_j + (B\dvec)_j }_{b}) z + \lambda | \underbrace{\wvec_j + \dvec_j}_{c} + z| 
 = - c + S \left( c - \frac{b}{a}, \frac{\lambda}{a} \right). 
\end{align}
Moreover, the low-rank structure of  $ B_k $  enables fast computation for coordinate update. 
As we describe in Appendix~\ref{app:pqncd}, the precomputation time is only $ O(m^2 p + m^3) $ where $ m \ll p $ is the memory size of L-BFGS, and each coordinate update is $ O(m) $. 
Furthermore, since we are using sparsity regularization, we can further speed up the algorithm by aggressively shrinking the active set of coordinates based on their subgradients~\citep{zhong2014}, and exclude the remaining dimensions from being updated. With the updates restricted to the active set $ \mathcal{S} $, all dependencies of the complexity on $ O(p) $ becomes $ O(|\mathcal{S}| ) $, which is substantially smaller. 
Hence the overall complexity of L-BFGS update is $ O(m^2 |\mathcal{S}| + m^3 + m |\mathcal{S}| T) $, where $ T $ is the number of inner iterations, typically $ T=10 $.

\subsection{Thresholding}
\label{sec:opt:feasibility}

In regression problems, it is known that post-processing estimates of coefficients via hard thresholding provably reduces the number of false discoveries \citep{zhou2009,wang2016}. Motivated by these encouraging results, we threshold the edge weights as follows: After obtaining a stationary point $\estecp$ of \eqref{eqn:augmented-lagrangian}, given a fixed threshold $\omega>0$, set any weights smaller than $\omega$ in absolute value to zero.
This strategy also has the important effect of ``rounding''  the numerical solution of the augmented Lagrangian \eqref{eqn:augmented-lagrangian}, since due to numerical precisions the solution satisfies $ h(\estecp) \le \epsilon $ for some small tolerance $ \epsilon $ near machine precision (e.g. $\epsilon=10^{-8}$), rather than $h(\estecp)=0$ strictly. 
However,
since $ h(\estecp) $ explicitly \emph{quantifies} the ``DAG-ness'' of $\estecp$ (see desiderata (b), Section~\ref{sec:acyclicity}), a small threshold $ \omega $ suffices to rule out cycle-inducing edges.

\section{Experiments}
\label{sec:exp}

We compared our method against greedy equivalent search (GES) \citep{chickering2003,ramsey2016}, the PC algorithm \citep{spirtes2000}, and LiNGAM \citep{shimizu2006}. \update{For GES, we used the fast greedy search (FGS) implementation from \citet{ramsey2016}.} Since the accuracy of PC and LiNGAM was significantly lower than either FGS or $\method$, we only report the results against FGS here. This is consistent with previous work on score-based learning \citep{aragam2015}, which also indicates that FGS outperforms other techniques such as hill-climbing and MMHC \citep{tsamardinos2006}. FGS was chosen since it is a state-of-the-art algorithm that scales to large problems.

For brevity, we outline the basic set-up of our experiments here; precise details of our experimental set-up, including all parameter choices and more detailed evaluations, can be found in Appendix~\ref{app:exp}. In each experiment, a random graph $\gr$ was generated from one of two random graph models, Erd\"os-R\'enyi (ER) or scale-free (SF). Given $\gr$, we assigned uniformly random edge weights to obtain a weight matrix $\true$. Given $\true$, we sampled $X=\true^{T}X+z\in\R^{d}$ from three different noise models: Gaussian ($\Gauss$), Exponential ($\Exp$), and Gumbel ($\Gumbel$). Based on these models, we generated random datasets $\dat\in\R^{n\times d}$ by generating rows i.i.d. according to one of these three models with $d\in\{10,20,50,100\}$ and $n\in\{20, 1000\}$. Since FGS outputs a CPDAG instead of a DAG or weight matrix, some care needs to be taken in making comparisons; see Appendix~\ref{app:exp:details} for details.

\subsection{Parameter estimation}
\label{sec:exp:est}

\begin{figure}[t]
\centering
\begin{subfigure}[t]{0.138\textwidth}
\centering
\includegraphics[width=0.99\textwidth]{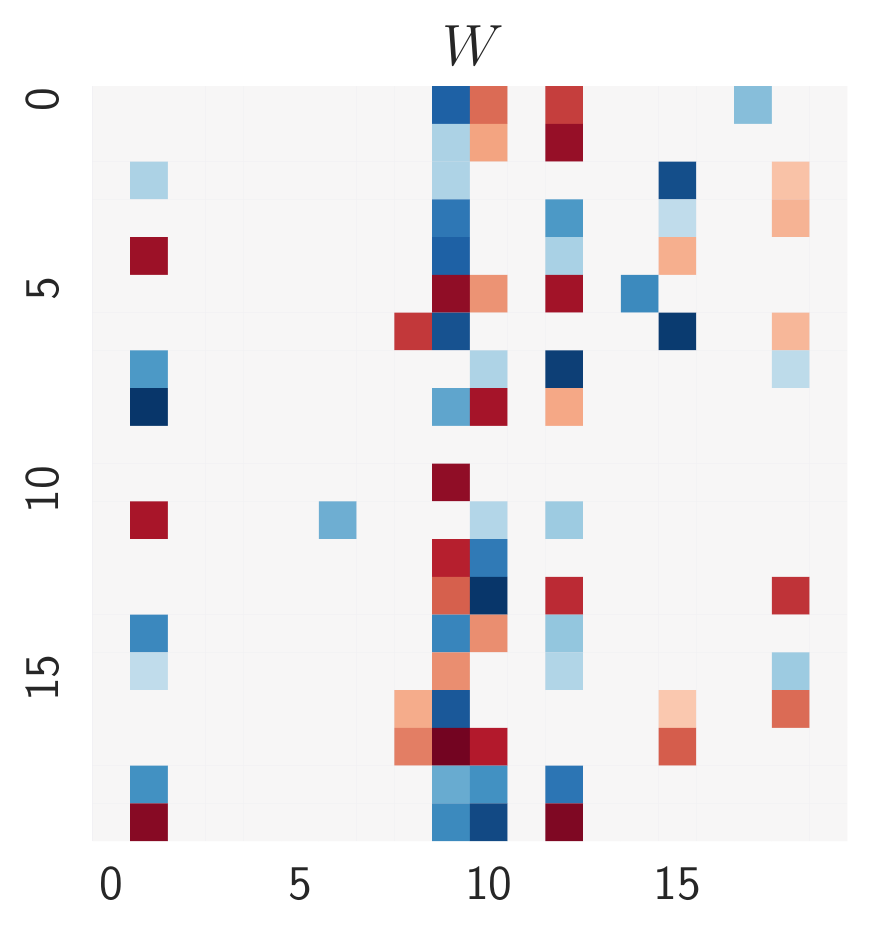}
\caption{true graph}
\end{subfigure}%
~
\begin{subfigure}[t]{0.43\textwidth}
\centering
\includegraphics[width=0.99\textwidth]{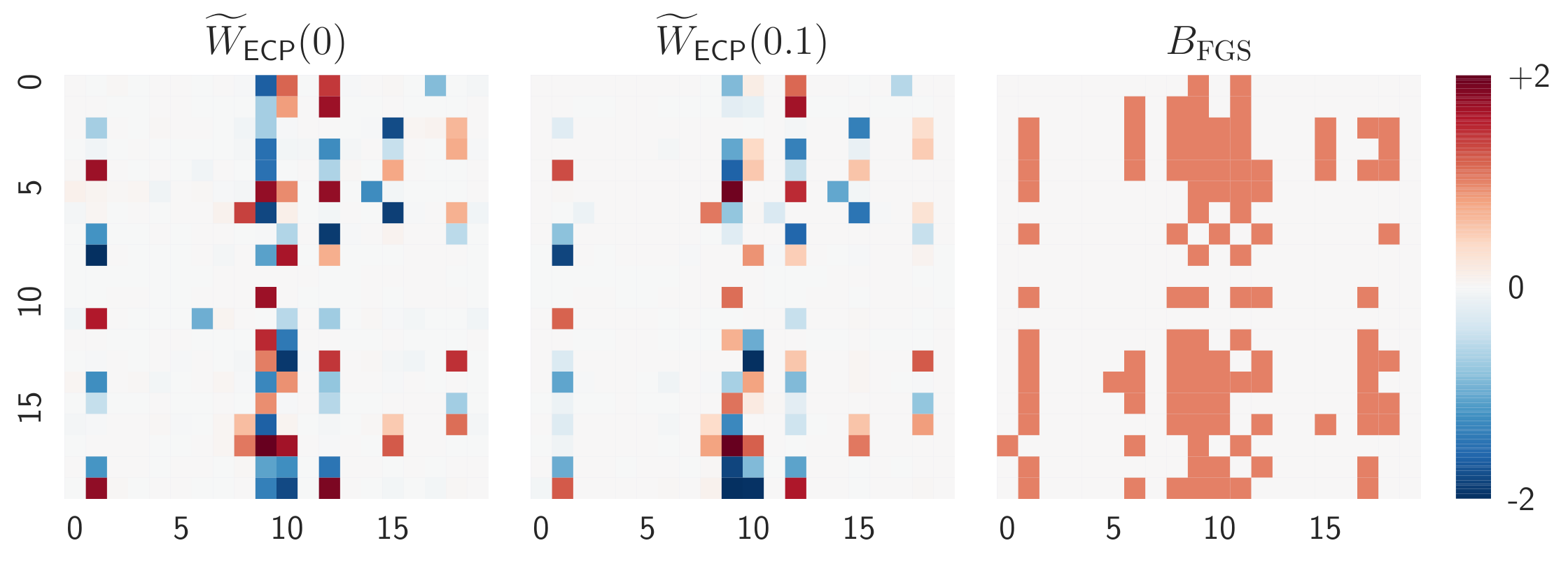}
\caption{estimate with $n=1000$}
\end{subfigure}%
\begin{subfigure}[t]{0.43\textwidth}
\centering
\includegraphics[width=0.99\textwidth]{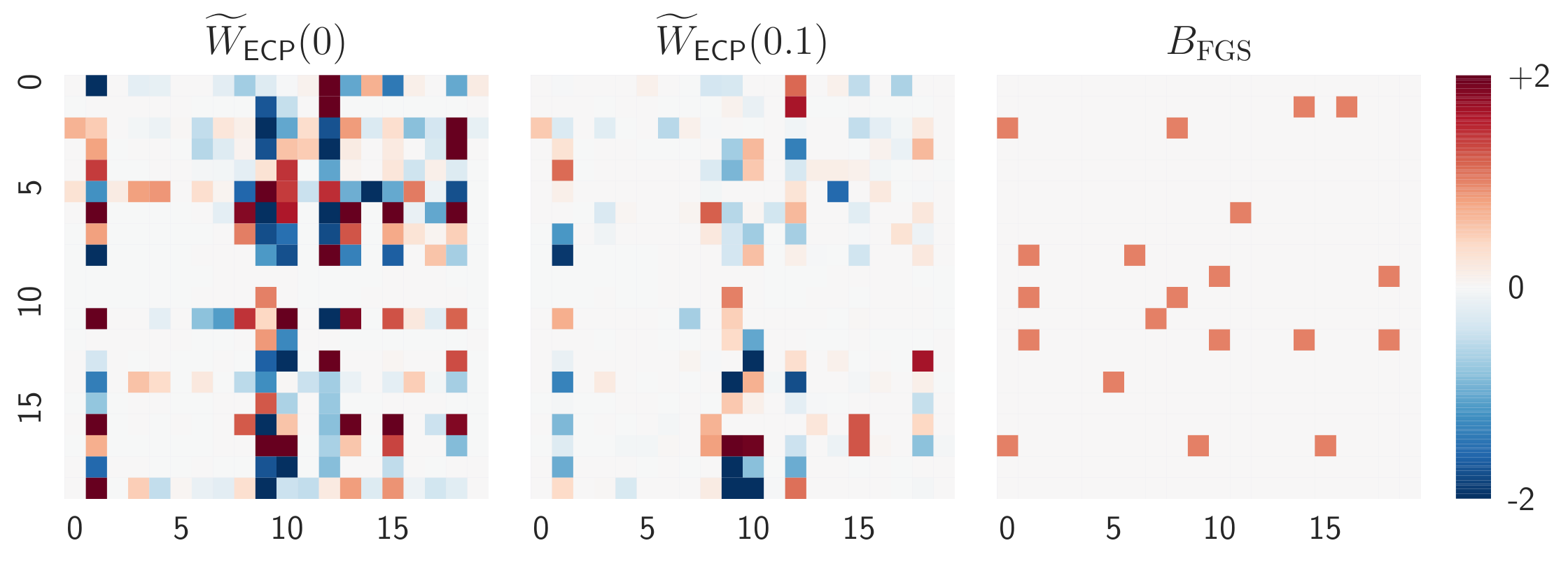}
\caption{estimate with $ n=20 $}
\end{subfigure}
\caption{Parameter estimates of $\estecp$ on a scale-free graph. Without the additional thresholding step in Algorithm~\ref{alg:notears}, $\method$ still produces consistent estimates of the true graph. The proposed method estimates the weights very well with large samples even without regularization, and remains accurate on insufficient samples when $\ell_{1}$-regularization is introduced. See also Figure~\ref{fig:compare:heatmap:er2}.}
\label{fig:compare:heatmap:sf4}
\end{figure}

We first performed a qualitative study of the solutions obtained by $\method$ \emph{without thresholding} by visualizing the weight matrix $\estecp$ obtained by solving $(\ecp)$ (i.e. $\omega=0$). 
This is illustrated in Figures~\ref{fig:compare:heatmap:er2} (ER-2) and~\ref{fig:compare:heatmap:sf4} (SF-4). 
The key takeaway is that our method provides (empirically) consistent parameter estimates of the true weight matrix $\true$. The final thresholding step in Algorithm~\ref{alg:notears} is only needed to ensure accuracy in structure learning. 
It also shows how effective is $ \ell_1 $-regularization in small $ n $ regime.

\subsection{Structure learning}
\label{sec:exp:struct}

We now examine our method for structure recovery, which is shown in Figure~\ref{fig:compare:shd}. 
For brevity, we only report the numbers for the structural Hamming distance (SHD) here, but complete figures and tables for additional metrics can be found in the supplement. 
Consistent with previous work on greedy methods, FGS is very competitive when
the number of edges is small (ER-2), but rapidly deterioriates
for even modest numbers of edges (SF-4). 
In the latter regime, $\method$ shows significant improvements.
This is consistent across each metric we evaluated, and the difference grows as the number of nodes $ d $ gets larger. 
Also notice that our algorithm performs uniformly better for each noise model ($\Exp$, $\Gauss$, and $\Gumbel$), without leveraging any specific knowledge about the noise type. 
Again, $ \ell_1 $-regularizer helps significantly in the small $ n $ setting.

\begin{figure}[t]
\centering
\begin{subfigure}[t]{0.48\textwidth}
\centering
\includegraphics[width=0.99\textwidth]{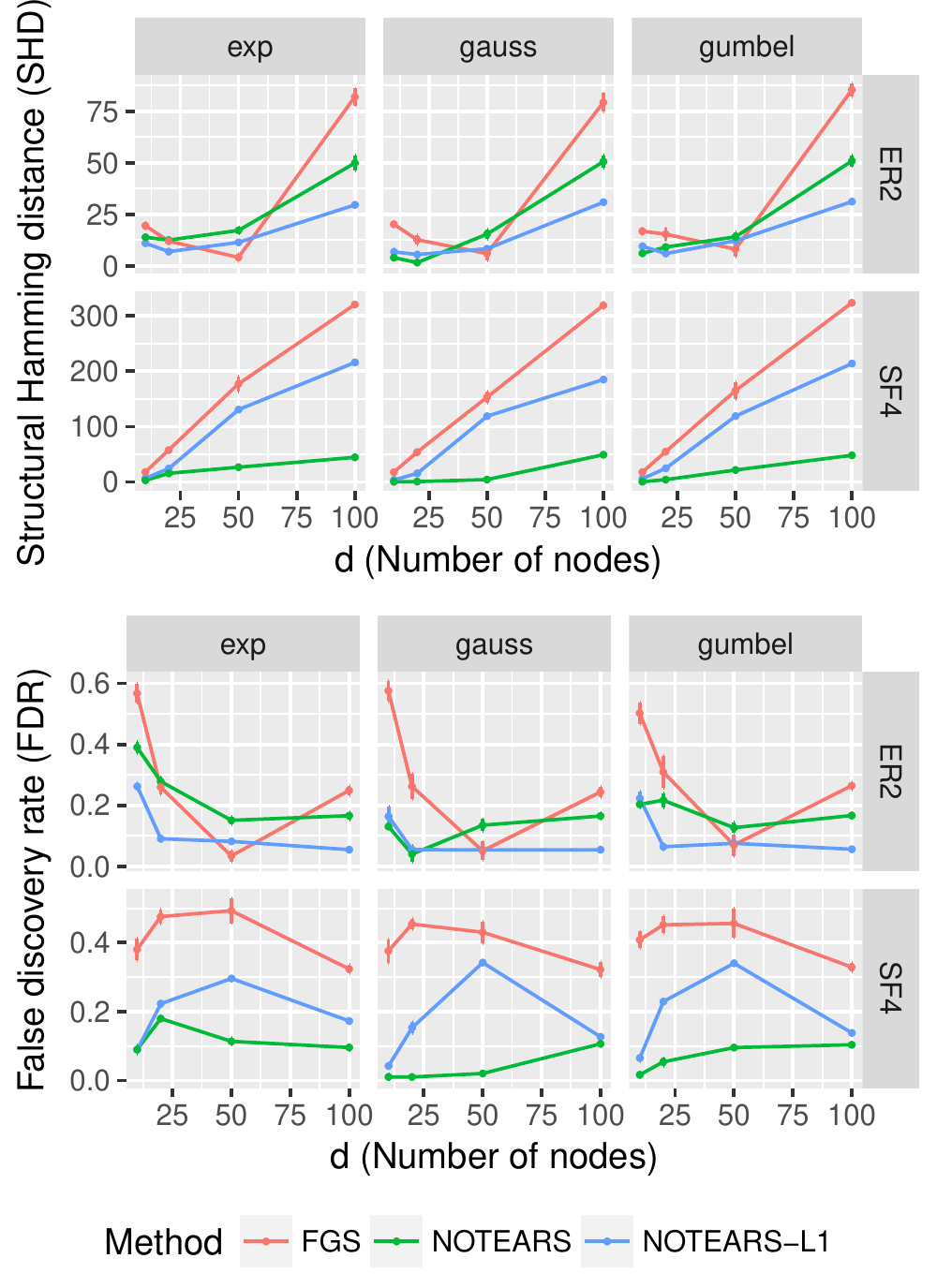}
\caption{SHD with $n=1000$}
\end{subfigure}%
\hspace{1em}
\begin{subfigure}[t]{0.48\textwidth}
\centering
\includegraphics[width=0.99\textwidth]{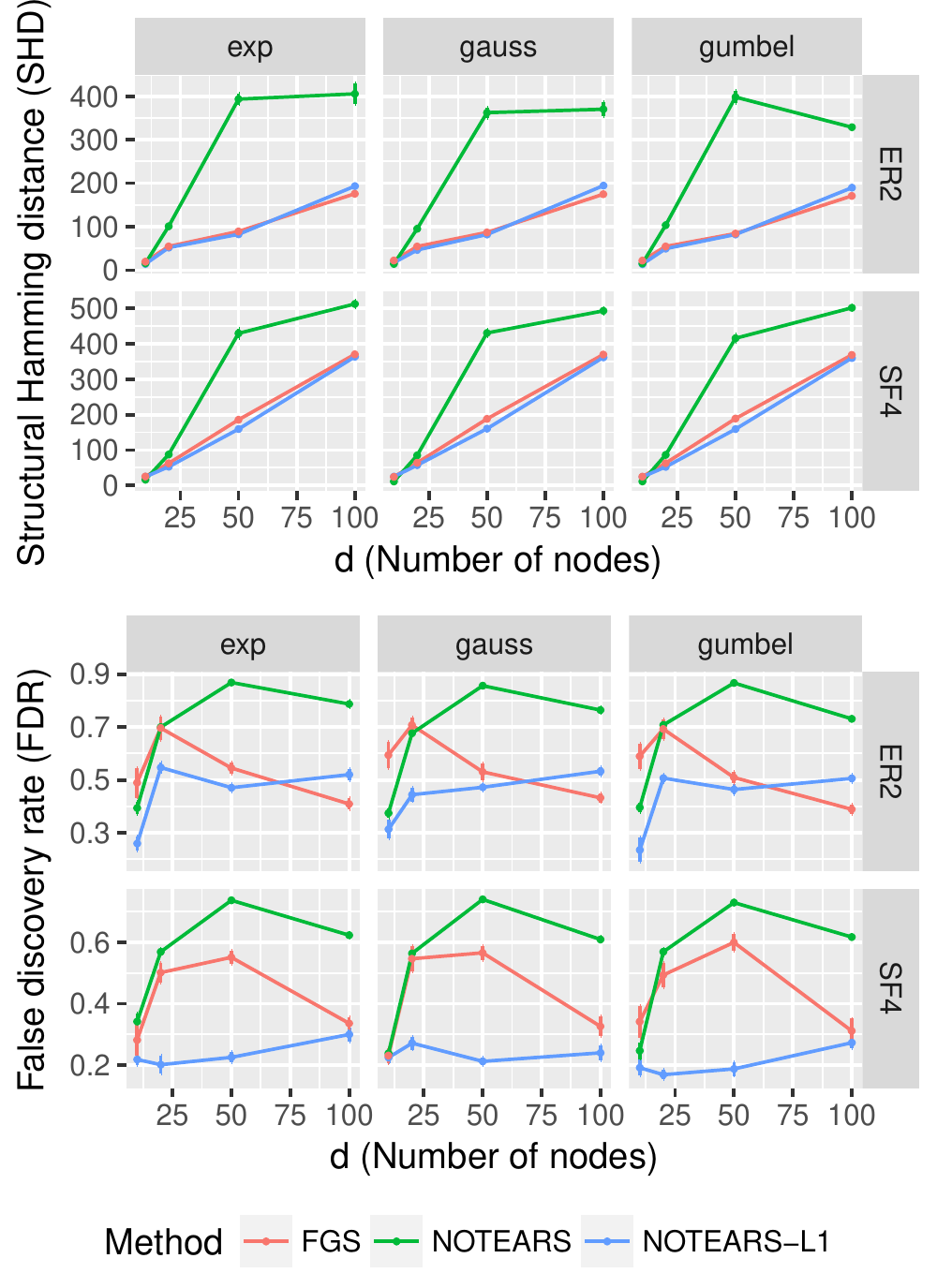}
\caption{SHD with $ n=20 $}
\end{subfigure}
\caption{Structure recovery in terms of SHD and FDR to the true graph (lower is better). 
Rows: random graph types, \{ER,SF\}-$ k $ = \{Erd\"os-R\'enyi, scale-free\} graphs with $ kd $ expected edges.
Columns: noise types of SEM. Error bars represent standard errors over 10 simulations.}
\label{fig:compare:shd}
\end{figure}

\subsection{Comparison to exact global minimizer}
\label{sec:exp:global}

In order to assess the ability of our method to solve the original program given by \eqref{eq:exact:program}, we used the GOBNILP program \citep{cussens2012,cussens2017} to find the exact minimizer of \eqref{eq:exact:program}. Since this involves enumerating all possible parent sets for each node, these experiments are limited to small DAGs. Nonetheless, these small-scale experiments yield valuable insight into how well $\method$ performs in actually solving the original problem. In our experiments we generated random graphs with $d=10$, and then generated 10 simulated datasets containing $n=20$ samples (for high-dimensions) and $n=1000$ (for low-dimensions). We then compared the scores returned by our method to the exact global minimizer computed by GOBNILP along with the estimated parameters. The results are shown in Table~\ref{tab:exact:er2}. Surprisingly, although $\method$ is only guaranteed to return a local minimizer, in many cases the obtained solution is very close to the global minimizer, as evidenced by deviations $\norm{\est-W_{\gr}}$. Since the general structure learning problem is NP-hard, we suspect that although the models we have tested (i.e. ER and SF) appear amenable to fast solution, in the worst-case there are graphs which will still take exponential time to run or get stuck in a local minimum. Furthermore, the problem becomes more difficult as $d$ increases. Nonetheless, this is encouraging evidence that the nonconvexity of \eqref{eq:ideal:program} is a minor issue in practice. We leave it to future work to investigate these problems further.

\begin{table}[t]
\centering
\caption{Comparison of $\method$ vs. globally optimal solution. $\Delta(\estgob,\est)=\score(\estgob) - \score(\est)$.}
\addtolength{\tabcolsep}{-2.5pt} 
\begin{tabular}{rrrrrrrrrr}
\toprule
$n$ & $\lambda$ & Graph & $\score(\true)$ & $\score(\estgob)$ & $\score(\est)$ & $\score(\estecp)$ & $\Delta(\estgob,\est)$ & $\norm{\est-W_{\gr}}$ & $\norm{\true-W_{\gr}}$ \\ 
\midrule
20 & 0 & ER2 & 5.11 & 3.85 & 5.36 & 3.88 & -1.52 & 0.07 & 3.38 \\ 
20 & 0.5 & ER2 & 16.04 & 12.81 & 13.49 & 12.90 & -0.68 & 0.12 & 3.15 \\ 
1000 & 0 & ER2 & 4.99 & 4.97 & 5.02 & 4.95 & -0.05 & 0.02 & 0.40 \\ 
1000 & 0.5 & ER2 & 15.93 & 13.32 & 14.03 & 13.46 & -0.71 & 0.12 & 2.95 \\ 
\midrule
20 & 0 & SF4 & 4.99 & 3.77 & 4.70 & 3.85 & -0.93 & 0.08 & 3.31 \\ 
20 & 0.5 & SF4 & 23.33 & 16.19 & 17.31 & 16.69 & -1.12 & 0.15 & 5.08 \\ 
1000 & 0 & SF4 & 4.96 & 4.94 & 5.05 & 4.99 & -0.11 & 0.04 & 0.29 \\ 
1000 & 0.5 & SF4 & 23.29 & 17.56 & 19.70 & 18.43 & -2.13 & 0.13 & 4.34 \\ 
\bottomrule
\end{tabular}
\addtolength{\tabcolsep}{2.5pt}
\label{tab:exact:er2}
\end{table}

\subsection{Real-data}
\label{sec:exp:real}

We also compared FGS and $\method$ on a real dataset provided by \citet{sachs2005}. This dataset consists of continuous measurements of expression levels of proteins and phospholipids in human immune system cells ($n = 7466$ $d = 11$, 20 edges). This dataset is a common benchmark in graphical models since it comes with a known \emph{consensus network}, that is, a gold standard network based on experimental annotations that is widely accepted by the biological community. In our experiments, FGS estimated 17 total edges with an SHD of 22, compared to 16 for $\method$ with an SHD of 22. 

\section{Discussion}
\label{sec:discussion}

\update{We have proposed a new method for learning DAGs from data based on a continuous optimization program. This represents a significant departure from existing approaches that search over the discrete space of DAGs, resulting in a difficult optimization program. We also proposed two optimization schemes for solving the resulting program to stationarity, and illustrated its advantages over existing methods such as greedy equivalence search. Crucially, by performing global updates (e.g. all parameters at once) instead of local updates (e.g. one edge at a time) in each iteration, our method is able to avoid relying on assumptions about the local structure of the graph. To conclude, let us discuss some of the limitations of our method and possible directions for future work.}

First, it is worth emphasizing once more that the equality constrained program \eqref{eq:ideal:program} is a nonconvex program. Thus, although we overcome the difficulties of \emph{combinatorial} optimization, our formulation still inherits the difficulties associated with \emph{nonconvex} optimization. In particular, black-box solvers can at best find stationary points of \eqref{eq:ideal:program}. With the exception of exact methods, however, existing methods suffer from this drawback as well.\footnote{GES \citep{chickering2003} is known to find the global minimizer in the limit $n\to\infty$ under certain assumptions, but this is not guaranteed for finite samples.} The main advantage of $\method$ then is \emph{smooth}, \emph{global search}, as opposed to combinatorial, local search; and furthermore the search is delegated to standard numerical solvers. 

\update{
Second, 
the current work relies on the smoothness of the score function, in order to make use of gradient-based numerical solvers to guide the graph search. 
However it is also interesting to consider non-smooth, even discrete scores such as BDe~\citep{heckerman1995}. 
Off-the-shelf techniques such as Nesterov's smoothing~\citep{nesterov2005smooth} could be useful, however more thorough investigation is left for future work.
}

Third, since the evaluation of the matrix exponential is $O(d^{3})$, the computational complexity of our method is cubic in the number of nodes, although the constant is small for sparse matrices. In fact, this is one of the key motivations for our use of second-order methods (as opposed to first-order), i.e. to reduce the number of matrix exponential computations. By using second-order methods, each iteration make significantly more progress than first-order methods. Furthermore, although in practice not many iterations ($t\sim10$) are required, we have not established any worst-case iteration complexity results. In light of the results in Section~\ref{sec:exp:global}, we expect there are exceptional cases where convergence is slow. Notwithstanding, $\method$ already outperforms existing methods when the in-degree is large, which is known difficult spot for existing methods. We leave it to future work to study these cases in more depth.

Lastly, in our experiments, we chose a fixed, suboptimal value of $\omega>0$ for thresholding (Section~\ref{sec:opt:feasibility}). Clearly, it would be preferable to find a data-driven choice of $\omega$ that adapts to different noise-to-signal ratios and graph types. It is an intersting direction for future to study such choices.

\update{The code is publicly available at \texttt{\url{https://github.com/xunzheng/notears}}.}

\bibliographystyle{abbrvnat}
\bibliography{dagpqnbib,zxbib}

\begin{thebibliography}{60}
\providecommand{\natexlab}[1]{#1}
\providecommand{\url}[1]{\texttt{#1}}
\expandafter\ifx\csname urlstyle\endcsname\relax
  \providecommand{\doi}[1]{doi: #1}\else
  \providecommand{\doi}{doi: \begingroup \urlstyle{rm}\Url}\fi

\bibitem[Al-Mohy and Higham(2009)]{almohy2009new}
A.~H. Al-Mohy and N.~J. Higham.
\newblock {A New Scaling and Squaring Algorithm for the Matrix Exponential}.
\newblock \emph{SIAM Journal on Matrix Analysis and Applications}, 2009.

\bibitem[Aragam and Zhou(2015)]{aragam2015}
B.~Aragam and Q.~Zhou.
\newblock Concave penalized estimation of sparse {G}aussian {B}ayesian
  networks.
\newblock \emph{Journal of Machine Learning Research}, 16:\penalty0 2273--2328,
  2015.

\bibitem[Aragam et~al.(2016)Aragam, Amini, and Zhou]{aragam2016}
B.~Aragam, A.~A. Amini, and Q.~Zhou.
\newblock Learning directed acyclic graphs with penalized neighbourhood
  regression.
\newblock \emph{Submitted}, arXiv:1511.08963, 2016.

\bibitem[Banerjee et~al.(2008)Banerjee, El~Ghaoui, and
  d'Aspremont]{banerjee2008}
O.~Banerjee, L.~El~Ghaoui, and A.~d'Aspremont.
\newblock Model selection through sparse maximum likelihood estimation for
  multivariate {G}aussian or binary data.
\newblock \emph{Journal of Machine Learning Research}, 9:\penalty0 485--516,
  2008.

\bibitem[Barab{\'a}si and Albert(1999)]{barabasi1999}
A.-L. Barab{\'a}si and R.~Albert.
\newblock Emergence of scaling in random networks.
\newblock \emph{Science}, 286\penalty0 (5439):\penalty0 509--512, 1999.

\bibitem[Botev et~al.(2017)Botev, Ritter, and Barber]{botev2017}
A.~Botev, H.~Ritter, and D.~Barber.
\newblock Practical gauss-newton optimisation for deep learning.
\newblock \emph{arXiv preprint arXiv:1706.03662}, 2017.

\bibitem[Bottou et~al.(2016)Bottou, Curtis, and Nocedal]{bottou2016}
L.~Bottou, F.~E. Curtis, and J.~Nocedal.
\newblock Optimization methods for large-scale machine learning.
\newblock \emph{arXiv preprint arXiv:1606.04838}, 2016.

\bibitem[Bouckaert(1993)]{bouckaert1993}
R.~R. Bouckaert.
\newblock Probabilistic network construction using the minimum description
  length principle.
\newblock In \emph{European conference on symbolic and quantitative approaches
  to reasoning and uncertainty}, pages 41--48. Springer, 1993.

\bibitem[Bousquet and Bottou(2008)]{bousquet2008}
O.~Bousquet and L.~Bottou.
\newblock The tradeoffs of large scale learning.
\newblock In \emph{Advances in neural information processing systems}, pages
  161--168, 2008.

\bibitem[Byrd et~al.(1995)Byrd, Lu, Nocedal, and Zhu]{byrd1995limited}
R.~H. Byrd, P.~Lu, J.~Nocedal, and C.~Zhu.
\newblock {A limited memory algorithm for bound constrained optimization}.
\newblock \emph{SIAM Journal on Scientific Computing}, 1995.

\bibitem[Chen et~al.(2016)Chen, Shen, Choi, and Darwiche]{chen2016bn}
E.~Y.-J. Chen, Y.~Shen, A.~Choi, and A.~Darwiche.
\newblock Learning bayesian networks with ancestral constraints.
\newblock In \emph{Advances in Neural Information Processing Systems}, pages
  2325--2333, 2016.

\bibitem[Chickering(1996)]{chickering1996}
D.~M. Chickering.
\newblock Learning {B}ayesian networks is {NP}-complete.
\newblock In \emph{Learning from data}, pages 121--130. Springer, 1996.

\bibitem[Chickering(2003)]{chickering2003}
D.~M. Chickering.
\newblock Optimal structure identification with greedy search.
\newblock \emph{Journal of Machine Learning Research}, 3:\penalty0 507--554,
  2003.

\bibitem[Chickering and Heckerman(1997)]{chickering1997}
D.~M. Chickering and D.~Heckerman.
\newblock Efficient approximations for the marginal likelihood of {B}ayesian
  networks with hidden variables.
\newblock \emph{Machine Learning}, 29\penalty0 (2-3):\penalty0 181--212, 1997.

\bibitem[Chickering et~al.(2004)Chickering, Heckerman, and
  Meek]{chickering2004}
D.~M. Chickering, D.~Heckerman, and C.~Meek.
\newblock Large-sample learning of {B}ayesian networks is {NP}-hard.
\newblock \emph{Journal of Machine Learning Research}, 5:\penalty0 1287--1330,
  2004.

\bibitem[Cussens(2012)]{cussens2012}
J.~Cussens.
\newblock Bayesian network learning with cutting planes.
\newblock \emph{arXiv preprint arXiv:1202.3713}, 2012.

\bibitem[Cussens et~al.(2017)Cussens, Haws, and Studen{\`y}]{cussens2017}
J.~Cussens, D.~Haws, and M.~Studen{\`y}.
\newblock Polyhedral aspects of score equivalence in bayesian network structure
  learning.
\newblock \emph{Mathematical Programming}, 164\penalty0 (1-2):\penalty0
  285--324, 2017.

\bibitem[Ellis and Wong(2008)]{ellis2008}
B.~Ellis and W.~H. Wong.
\newblock Learning causal {B}ayesian network structures from experimental data.
\newblock \emph{Journal of the American Statistical Association}, 103\penalty0
  (482), 2008.

\bibitem[Friedman et~al.(2008)Friedman, Hastie, and Tibshirani]{friedman2008}
J.~Friedman, T.~Hastie, and R.~Tibshirani.
\newblock Sparse inverse covariance estimation with the {G}raphical {L}asso.
\newblock \emph{Biostatistics}, 9\penalty0 (3):\penalty0 432--441, 2008.

\bibitem[Fu and Zhou(2013)]{fu2013}
F.~Fu and Q.~Zhou.
\newblock Learning sparse causal {G}aussian networks with experimental
  intervention: {R}egularization and coordinate descent.
\newblock \emph{Journal of the American Statistical Association}, 108\penalty0
  (501):\penalty0 288--300, 2013.

\bibitem[G{\'a}mez et~al.(2011)G{\'a}mez, Mateo, and Puerta]{gamez2011}
J.~A. G{\'a}mez, J.~L. Mateo, and J.~M. Puerta.
\newblock Learning {B}ayesian networks by hill climbing: {E}fficient methods
  based on progressive restriction of the neighborhood.
\newblock \emph{Data Mining and Knowledge Discovery}, 22\penalty0
  (1-2):\penalty0 106--148, 2011.

\bibitem[Grant and Boyd(2014)]{cvx}
M.~Grant and S.~Boyd.
\newblock {CVX}: Matlab software for disciplined convex programming, version
  2.1.
\newblock \url{http://cvxr.com/cvx}, Mar. 2014.

\bibitem[Gu et~al.(2018)Gu, Fu, and Zhou]{gu2018}
J.~Gu, F.~Fu, and Q.~Zhou.
\newblock Penalized estimation of directed acyclic graphs from discrete data.
\newblock \emph{Statistics and Computing}, DOI: 10.1007/s11222-018-9801-y,
  2018.

\bibitem[Harary and Manvel(1971)]{harary1971number}
F.~Harary and B.~Manvel.
\newblock {On the number of cycles in a graph}.
\newblock \emph{Matematick{\`y} {\v{c}}asopis}, 1971.

\bibitem[Heckerman et~al.(1995)Heckerman, Geiger, and
  Chickering]{heckerman1995}
D.~Heckerman, D.~Geiger, and D.~M. Chickering.
\newblock Learning {B}ayesian networks: {T}he combination of knowledge and
  statistical data.
\newblock \emph{Machine learning}, 20\penalty0 (3):\penalty0 197--243, 1995.

\bibitem[Hsieh et~al.(2014)Hsieh, Sustik, Dhillon, and
  Ravikumar]{hsieh2014quic}
C.-J. Hsieh, M.~A. Sustik, I.~S. Dhillon, and P.~Ravikumar.
\newblock Quic: quadratic approximation for sparse inverse covariance
  estimation.
\newblock \emph{Journal of Machine Learning Research}, 15\penalty0
  (1):\penalty0 2911--2947, 2014.

\bibitem[Kingma and Ba(2014)]{kingma2014}
D.~P. Kingma and J.~Ba.
\newblock Adam: A method for stochastic optimization.
\newblock \emph{arXiv preprint arXiv:1412.6980}, 2014.

\bibitem[Koller and Friedman(2009)]{koller2009}
D.~Koller and N.~Friedman.
\newblock \emph{Probabilistic graphical models: principles and techniques}.
\newblock MIT press, 2009.

\bibitem[Kuipers et~al.(2014)Kuipers, Moffa, and Heckerman]{kuipers2014}
J.~Kuipers, G.~Moffa, and D.~Heckerman.
\newblock Addendum on the scoring of gaussian directed acyclic graphical
  models.
\newblock \emph{The Annals of Statistics}, pages 1689--1691, 2014.

\bibitem[Loh and B{\"u}hlmann(2014)]{loh2014causal}
P.-L. Loh and P.~B{\"u}hlmann.
\newblock High-dimensional learning of linear causal networks via inverse
  covariance estimation.
\newblock \emph{Journal of Machine Learning Research}, 15:\penalty0 3065--3105,
  2014.

\bibitem[Nemirovski(1999)]{nemirovski1999optimization}
A.~Nemirovski.
\newblock {Optimization II: Standard Numerical Methods for Nonlinear Continuous
  Optimization}.
\newblock 1999.

\bibitem[Nesterov(2005)]{nesterov2005smooth}
Y.~Nesterov.
\newblock {Smooth minimization of non-smooth functions}.
\newblock \emph{Mathematical Programming}, 2005.

\bibitem[Niinim{\"a}ki et~al.(2016)Niinim{\"a}ki, Parviainen, and
  Koivisto]{niinimaki2016}
T.~Niinim{\"a}ki, P.~Parviainen, and M.~Koivisto.
\newblock Structure discovery in bayesian networks by sampling partial orders.
\newblock \emph{Journal of Machine Learning Research}, 17\penalty0
  (1):\penalty0 2002--2048, 2016.

\bibitem[Nocedal and Wright(2006)]{nocedal2006numerical}
J.~Nocedal and S.~J. Wright.
\newblock \emph{{Numerical Optimization}}.
\newblock 2006.

\bibitem[Ott and Miyano(2003)]{ott2003}
S.~Ott and S.~Miyano.
\newblock Finding optimal gene networks using biological constraints.
\newblock \emph{Genome Informatics}, 14:\penalty0 124--133, 2003.

\bibitem[Ramsey et~al.(2016)Ramsey, Glymour, Sanchez-Romero, and
  Glymour]{ramsey2016}
J.~Ramsey, M.~Glymour, R.~Sanchez-Romero, and C.~Glymour.
\newblock A million variables and more: the fast greedy equivalence search
  algorithm for learning high-dimensional graphical causal models, with an
  application to functional magnetic resonance images.
\newblock \emph{International Journal of Data Science and Analytics}, pages
  1--9, 2016.

\bibitem[Robinson(1977)]{robinson1977}
R.~W. Robinson.
\newblock Counting unlabeled acyclic digraphs.
\newblock In \emph{Combinatorial mathematics V}, pages 28--43. Springer, 1977.

\bibitem[Sachs et~al.(2005)Sachs, Perez, Pe'er, Lauffenburger, and
  Nolan]{sachs2005}
K.~Sachs, O.~Perez, D.~Pe'er, D.~A. Lauffenburger, and G.~P. Nolan.
\newblock Causal protein-signaling networks derived from multiparameter
  single-cell data.
\newblock \emph{Science}, 308\penalty0 (5721):\penalty0 523--529, 2005.

\bibitem[Scanagatta et~al.(2015)Scanagatta, de~Campos, Corani, and
  Zaffalon]{scanagatta2015}
M.~Scanagatta, C.~P. de~Campos, G.~Corani, and M.~Zaffalon.
\newblock Learning bayesian networks with thousands of variables.
\newblock In \emph{Advances in Neural Information Processing Systems}, pages
  1864--1872, 2015.

\bibitem[Scanagatta et~al.(2016)Scanagatta, Corani, de~Campos, and
  Zaffalon]{scanagatta2016}
M.~Scanagatta, G.~Corani, C.~P. de~Campos, and M.~Zaffalon.
\newblock Learning treewidth-bounded bayesian networks with thousands of
  variables.
\newblock In \emph{Advances in Neural Information Processing Systems}, pages
  1462--1470, 2016.

\bibitem[Schmidt et~al.(2007)Schmidt, Niculescu-Mizil, and Murphy]{schmidt2007}
M.~Schmidt, A.~Niculescu-Mizil, and K.~Murphy.
\newblock Learning graphical model structure using {L1}-regularization paths.
\newblock In \emph{AAAI}, volume~7, pages 1278--1283, 2007.

\bibitem[Schmidt et~al.(2009)Schmidt, Berg, Friedlander, and
  Murphy]{schmidt2009}
M.~Schmidt, E.~Berg, M.~Friedlander, and K.~Murphy.
\newblock Optimizing costly functions with simple constraints: A limited-memory
  projected quasi-newton algorithm.
\newblock In \emph{Artificial Intelligence and Statistics}, pages 456--463,
  2009.

\bibitem[Shimizu et~al.(2006)Shimizu, Hoyer, Hyv{\"a}rinen, and
  Kerminen]{shimizu2006}
S.~Shimizu, P.~O. Hoyer, A.~Hyv{\"a}rinen, and A.~Kerminen.
\newblock A linear non-{G}aussian acyclic model for causal discovery.
\newblock \emph{Journal of Machine Learning Research}, 7:\penalty0 2003--2030,
  2006.

\bibitem[Silander and Myllymaki(2006)]{silander2012}
T.~Silander and P.~Myllymaki.
\newblock A simple approach for finding the globally optimal bayesian network
  structure.
\newblock In \emph{Proceedings of the 22nd Conference on Uncertainty in
  Artificial Intelligence}, 2006.

\bibitem[Singh and Moore(2005)]{singh2005}
A.~P. Singh and A.~W. Moore.
\newblock Finding optimal bayesian networks by dynamic programming.
\newblock 2005.

\bibitem[Spirtes and Glymour(1991)]{spirtes1991}
P.~Spirtes and C.~Glymour.
\newblock An algorithm for fast recovery of sparse causal graphs.
\newblock \emph{Social Science Computer Review}, 9\penalty0 (1):\penalty0
  62--72, 1991.

\bibitem[Spirtes et~al.(2000)Spirtes, Glymour, and Scheines]{spirtes2000}
P.~Spirtes, C.~Glymour, and R.~Scheines.
\newblock \emph{Causation, prediction, and search}, volume~81.
\newblock The MIT Press, 2000.

\bibitem[Taylor et~al.(2016)Taylor, Burmeister, Xu, Singh, Patel, and
  Goldstein]{taylor2016}
G.~Taylor, R.~Burmeister, Z.~Xu, B.~Singh, A.~Patel, and T.~Goldstein.
\newblock Training neural networks without gradients: A scalable admm approach.
\newblock In \emph{International Conference on Machine Learning}, pages
  2722--2731, 2016.

\bibitem[Teyssier and Koller(2005)]{teyssier2012}
M.~Teyssier and D.~Koller.
\newblock Ordering-based search: A simple and effective algorithm for learning
  bayesian networks.
\newblock In \emph{Uncertainty in Artifical Intelligence (UAI)}, 2005.

\bibitem[Tsamardinos et~al.(2006)Tsamardinos, Brown, and
  Aliferis]{tsamardinos2006}
I.~Tsamardinos, L.~E. Brown, and C.~F. Aliferis.
\newblock The max-min hill-climbing {B}ayesian network structure learning
  algorithm.
\newblock \emph{Machine Learning}, 65\penalty0 (1):\penalty0 31--78, 2006.

\bibitem[Van~Beek and Hoffmann(2015)]{van-beek2015}
P.~Van~Beek and H.-F. Hoffmann.
\newblock Machine learning of bayesian networks using constraint programming.
\newblock In \emph{International Conference on Principles and Practice of
  Constraint Programming}, pages 429--445. Springer, 2015.

\bibitem[van~de Geer and B{\"u}hlmann(2013)]{geer2013}
S.~van~de Geer and P.~B{\"u}hlmann.
\newblock $\ell_0$-penalized maximum likelihood for sparse directed acyclic
  graphs.
\newblock \emph{Annals of Statistics}, 41\penalty0 (2):\penalty0 536--567,
  2013.

\bibitem[Wang et~al.(2016)Wang, Dunson, and Leng]{wang2016}
X.~Wang, D.~Dunson, and C.~Leng.
\newblock No penalty no tears: Least squares in high-dimensional linear models.
\newblock In \emph{International Conference on Machine Learning}, pages
  1814--1822, 2016.

\bibitem[Watts and Strogatz(1998)]{watts1998}
D.~J. Watts and S.~H. Strogatz.
\newblock Collective dynamics of small-world networks.
\newblock \emph{nature}, 393\penalty0 (6684):\penalty0 440, 1998.

\bibitem[Xiang and Kim(2013)]{xiang2013}
J.~Xiang and S.~Kim.
\newblock A* {L}asso for learning a sparse {B}ayesian network structure for
  continuous variables.
\newblock In \emph{Advances in Neural Information Processing Systems}, pages
  2418--2426, 2013.

\bibitem[Yuan and Lin(2007)]{yuan2007}
M.~Yuan and Y.~Lin.
\newblock Model selection and estimation in the {G}aussian graphical model.
\newblock \emph{Biometrika}, 94\penalty0 (1):\penalty0 19--35, 2007.

\bibitem[Zhang et~al.(2013)Zhang, Gaiteri, Bodea, Wang, McElwee,
  Podtelezhnikov, Zhang, Xie, Tran, Dobrin, et~al.]{zhang2013}
B.~Zhang, C.~Gaiteri, L.-G. Bodea, Z.~Wang, J.~McElwee, A.~A. Podtelezhnikov,
  C.~Zhang, T.~Xie, L.~Tran, R.~Dobrin, et~al.
\newblock Integrated systems approach identifies genetic nodes and networks in
  late-onset alzheimer's disease.
\newblock \emph{Cell}, 153\penalty0 (3):\penalty0 707--720, 2013.

\bibitem[Zhong et~al.(2014)Zhong, Yen, Dhillon, and Ravikumar]{zhong2014}
K.~Zhong, I.~E.-H. Yen, I.~S. Dhillon, and P.~K. Ravikumar.
\newblock Proximal quasi-newton for computationally intensive l1-regularized
  m-estimators.
\newblock In \emph{Advances in Neural Information Processing Systems}, pages
  2375--2383, 2014.

\bibitem[Zhou(2011)]{zhou2011}
Q.~Zhou.
\newblock Multi-domain sampling with applications to structural inference of
  {B}ayesian networks.
\newblock \emph{Journal of the American Statistical Association}, 106\penalty0
  (496):\penalty0 1317--1330, 2011.

\bibitem[Zhou(2009)]{zhou2009}
S.~Zhou.
\newblock Thresholding procedures for high dimensional variable selection and
  statistical estimation.
\newblock In \emph{Advances in Neural Information Processing Systems}, pages
  2304--2312, 2009.

\end{thebibliography}

\newpage

\appendix

\section{Details of Proximal Quasi-Newton}\label{app:pqncd}

Recall $ B_k \in \R^{p \times p} $ is the low-rank approximation of the Hessian matrix given by L-BFGS updates.
Let the memory size of L-BFGS be $ m $, which is taken to be $ m \ll p $. 
The compact form of L-BFGS update can be written as 
\begin{align}
B_k = \gamma_k I - Q \widehat{Q},
\end{align}
where 
\begin{gather}
Q = \begin{bmatrix}
\gamma_k S_k & Y_k 
\end{bmatrix}, \ 
R = \begin{bmatrix}
\gamma_k S_k^T  S_k & L_k \\
L_k^T  & - D_k 
\end{bmatrix}^{-1}, \ 
\widehat{Q} = R Q^T,  \nonumber \\ 
S_k = \begin{bmatrix}
\svec_{k-m} & \cdots & \svec_{k-1}
\end{bmatrix}, \ 
Y_k = \begin{bmatrix}
\yvec_{k-m} & \cdots & \yvec_{k-1}
\end{bmatrix},   \nonumber \\
\svec_k  = \wvec_{k+1} - \wvec_k,  \
\yvec_k = \gvec_{k+1} - \gvec_k,   \
\gamma_{k} = {\yvec_{k-1}^T \yvec_{k-1}} / {\svec_{k-1}^T \yvec_{k-1}}, \nonumber \\
D_k = \text{diag} \begin{bmatrix}
\svec_{k-m}^T \yvec_{k-m} & \cdots & \svec_{k-1}^T \yvec_{k-1}
\end{bmatrix}, \
(L_k)_{ij} = \begin{cases}
\svec_{k-m+i-1}^T \yvec_{k-m+j-1}  & \text{if } i > j \\
0 & \text{otherwise}
\end{cases}.  \nonumber 
\end{gather}

The low rank structure of $ B_k $ enables fast computation of subsequent coordinate descent procedure. 
Specifically, notice that all $ Q, R, \widehat{Q} $, and $ \mathrm{diag}(B) $ can be precomputed in $ O(m^2p + m^3) $ time, which is significantly smaller than naive Hessian inversion $ O(p^3) $. 
After precomputation, in each coordinate update, both $ a $ and $ c $ in \eqref{eqn:coordinate-update} can be computed and updated in $ O(1) $ time. 
Moreover, let $ \widehat{\dvec} = \widehat{Q} \dvec \in \R^{2m} $,  
we have $ (B \dvec)_j = \gamma \dvec_j - Q_{j,:} \widehat{\dvec} $, which suggests $ b $ in \eqref{eqn:coordinate-update} only requires $ O(m) $ to compute and update. 
Therefore each coordinate update is $ O(m) $. 

The detailed procedure of PQN is outlined in Algorithm~\ref{alg:pqncd}. 

\begin{algorithm}[t]
\caption{Proximal Quasi-Newton for unconstrained problem~\citep{zhong2014}}
\label{alg:pqncd}
\begin{enumerate}
\item Input: $ \wvec_0 $, $ \gvec_0 = \nabla f(\wvec_0) $, active set $ \mathcal{S} = [p] $.
\item For $ k = 0, 1, 2, \dotsc$:
    \begin{enumerate}
    \item Shrink $ \mathcal{S} $ to rule out $ j $ with $ w_j = 0 $ or small subgradient $ |\partial_j L(\wvec)| $ 
    \item If shrinking stopping criteria is satisfied
        \begin{enumerate}
        \item Reset $ \mathcal{S} = [p] $ and L-BFGS memory
        \item Update shrinking stopping criteria and continue
        \end{enumerate}
    \item Solve \eqref{eqn:descent-direction-subproblem} for descent direction $ \dvec_k $ using coordinate update~\eqref{eqn:coordinate-update} on active set
    \item Line search for step size $ \eta \in (0,1] $ until Armijo rule is satisfied:
    \begin{align}
    f(\wvec_k + \eta \dvec_k) \le f(\wvec_k) + \eta c_1 (\lambda \norm{\wvec_{k} + \dvec_k}_1  - \lambda \norm{\wvec_{k}} + \gvec_k^T\dvec_k ), 
    \end{align}
    where $ c_1 $ is some small constant, typically set to $ 10^{-3} $ or $ 10^{-4} $.
    \item Generate new iterate $ \wvec_{k+1} \gets \wvec_{k} + \eta \dvec_k $
    \item Update $ \gvec, \svec, \yvec, Q, R, \widehat{Q} $ restricted to $ \mathcal{S} $
    \end{enumerate}
\end{enumerate}
\end{algorithm}

\section{Sensitivity of threshold}\label{app:sensitivity-of-threshold}

We demonstrate the effect of threshold in Figure~\ref{fig:compare:roc}. 
For each setting, we computed the ``ROC'' curve for FDR and TPR with varying level of threshold, while ensuring the resulting graph is indeed a DAG. 
On the right, we also present the estimated edge weights of $ \estecp $ in decreasing order. 
One can first observe that in all cases most of the edge weights are equal or close to zero as expected. 
The remaining question is how to choose a threshold that separates out these (near zero) from signals (away from zero) so that best performance can be achieved. 
With enough samples, one can often notice a sudden change in the weight distribution as in  Figure~\ref{fig:compare:roc}(a)(c). 
With insufficient samples, the breakpoint is less clear, and the optimal choice that balances between TPR and FDR is depends on the specific settings. 
Nonetheless, the predictive performance is less sensitive to threshold value as one can see from the slope of the decrease in the weights before getting close to zero. 
Indeed, in our experiments, we found a fixed threshold $ \omega = 0.3 $ is a suboptimal yet reasonable choice across many different settings.

\begin{figure}[t]
\centering
\begin{subfigure}[t]{0.49\textwidth}
\centering
\includegraphics[width=0.48\textwidth]{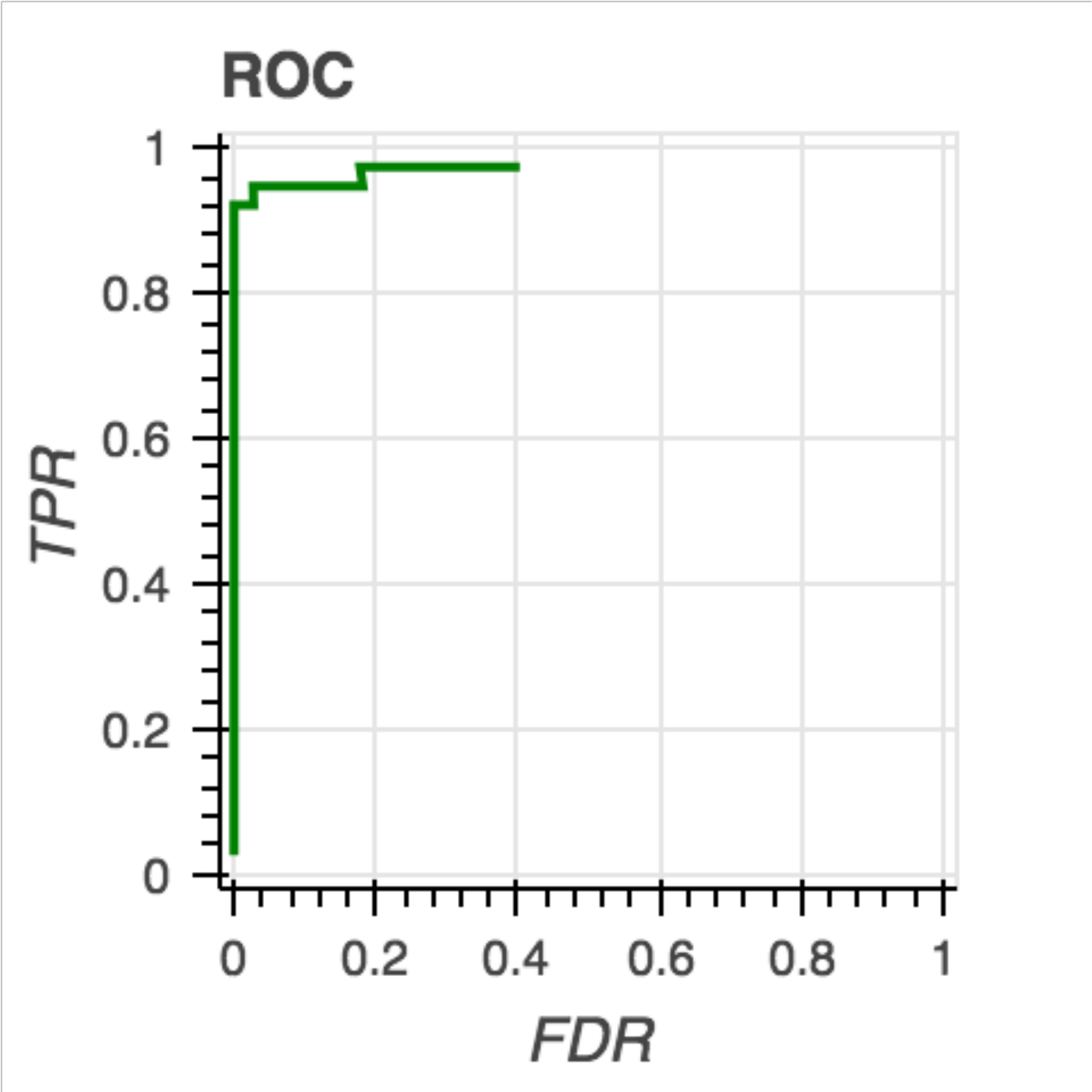}
\includegraphics[width=0.48\textwidth]{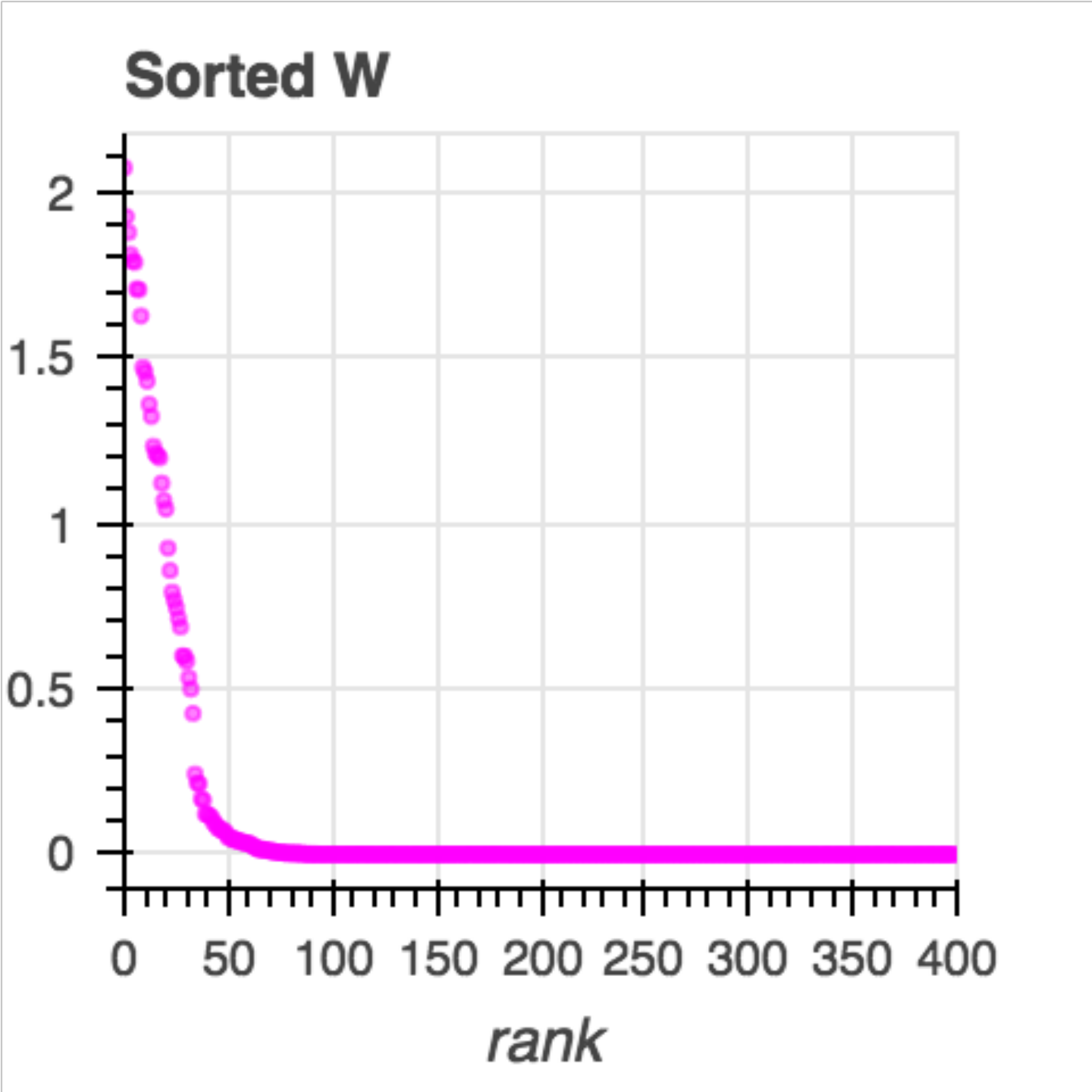}
\caption{ ER2, $ n= 1000 $ }
\end{subfigure}%
\begin{subfigure}[t]{0.49\textwidth}
\centering
\includegraphics[width=0.48\textwidth]{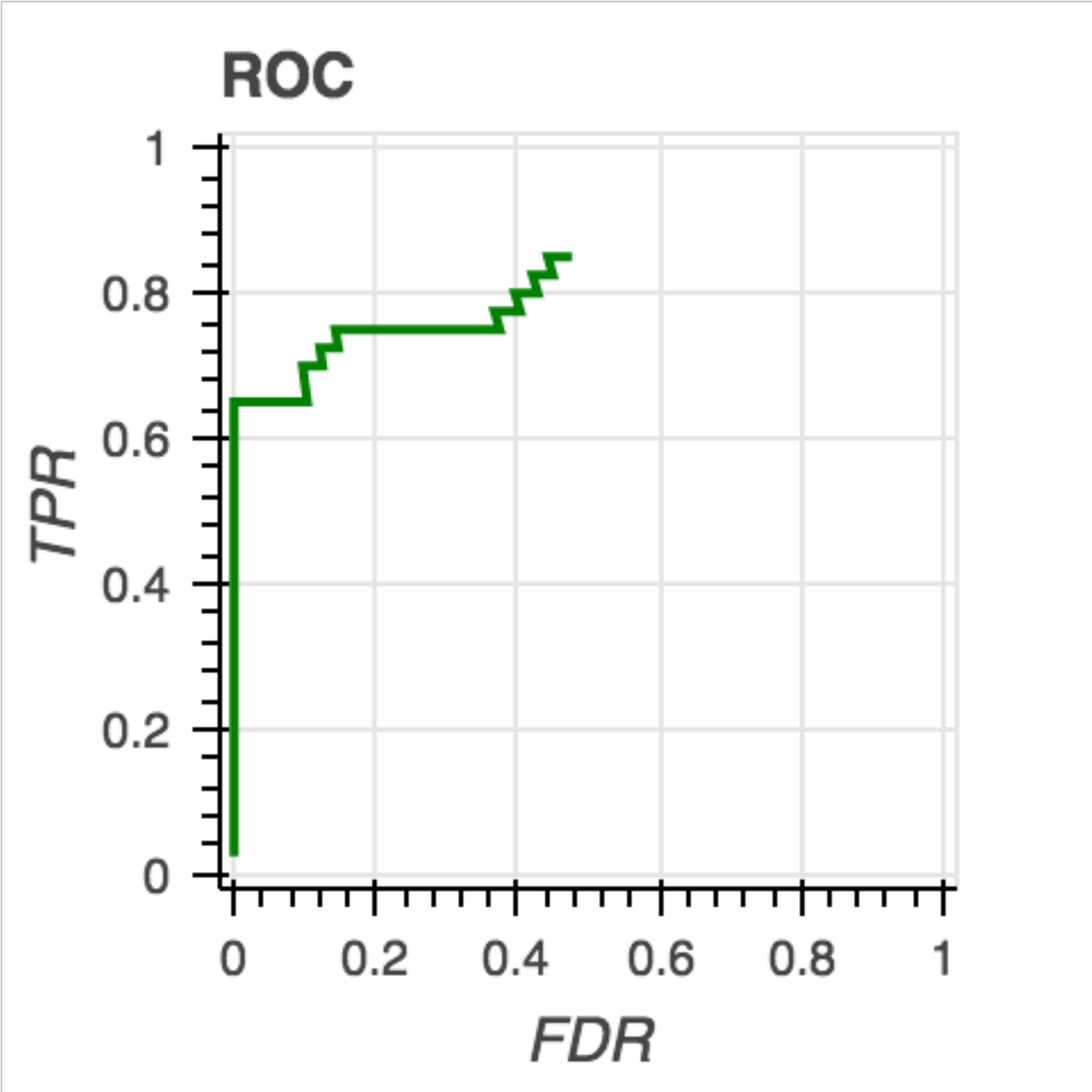}
\includegraphics[width=0.48\textwidth]{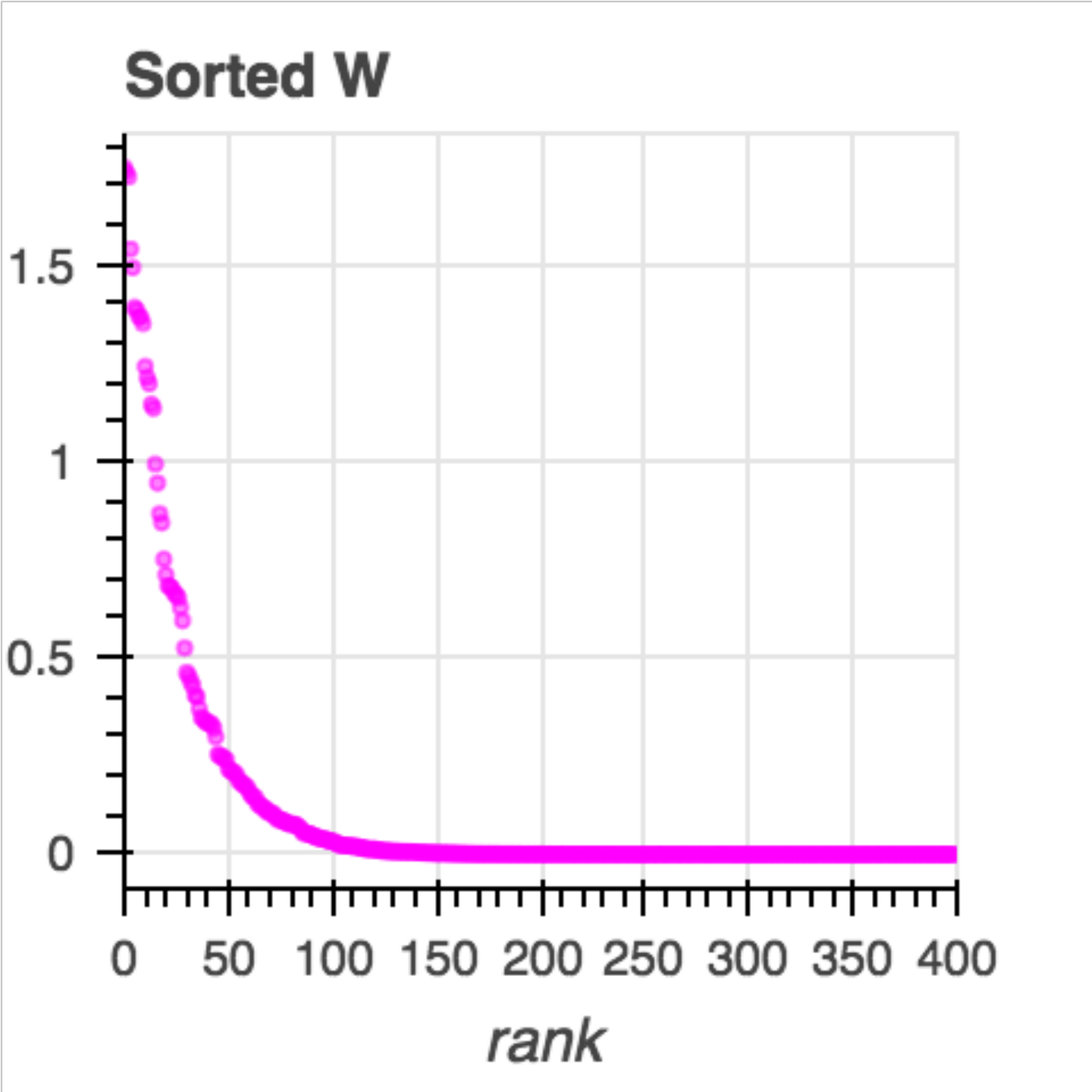}
\caption{ ER2, $ n= 20 $ }
\end{subfigure} 
\\
\begin{subfigure}[t]{0.49\textwidth}
\centering
\includegraphics[width=0.48\textwidth]{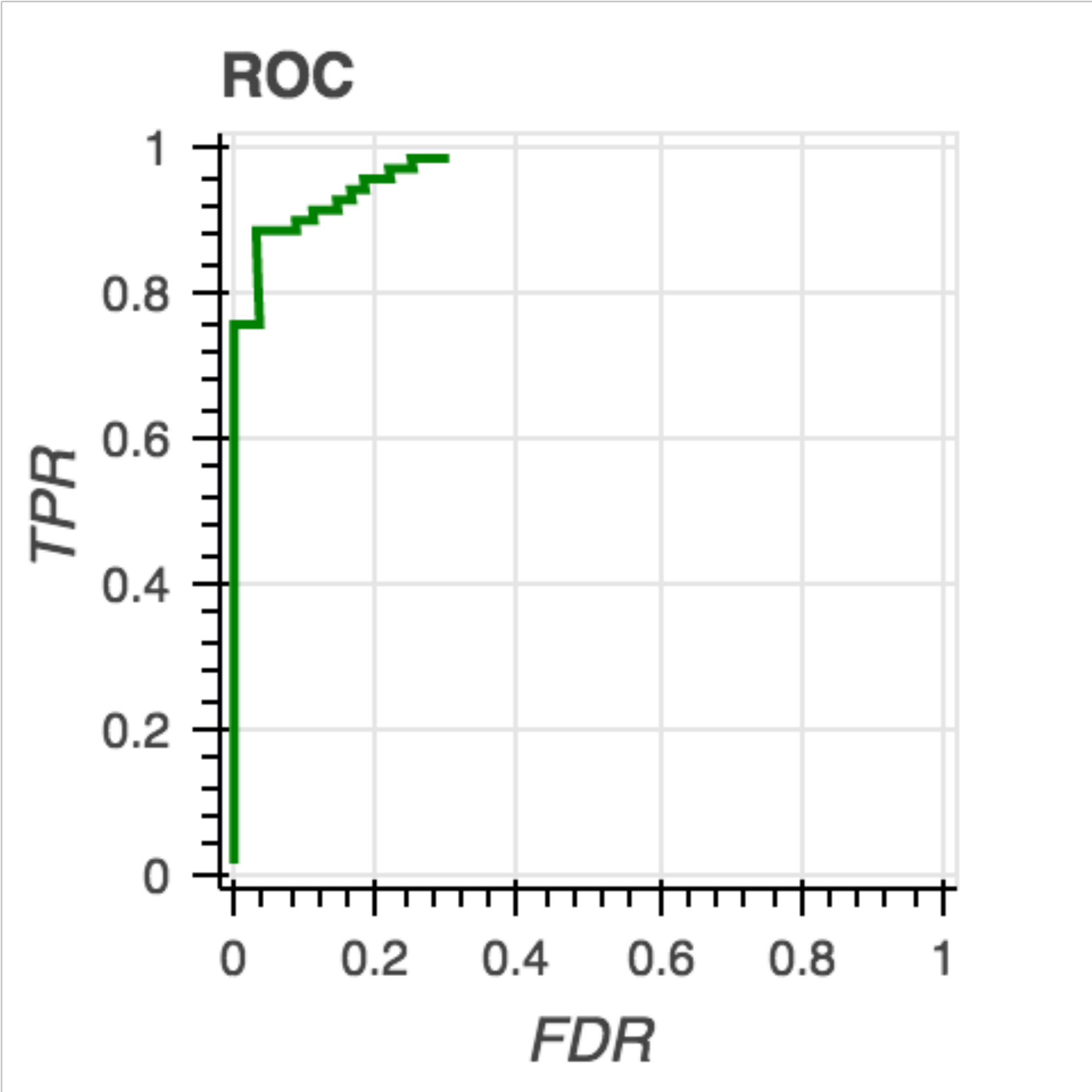}
\includegraphics[width=0.48\textwidth]{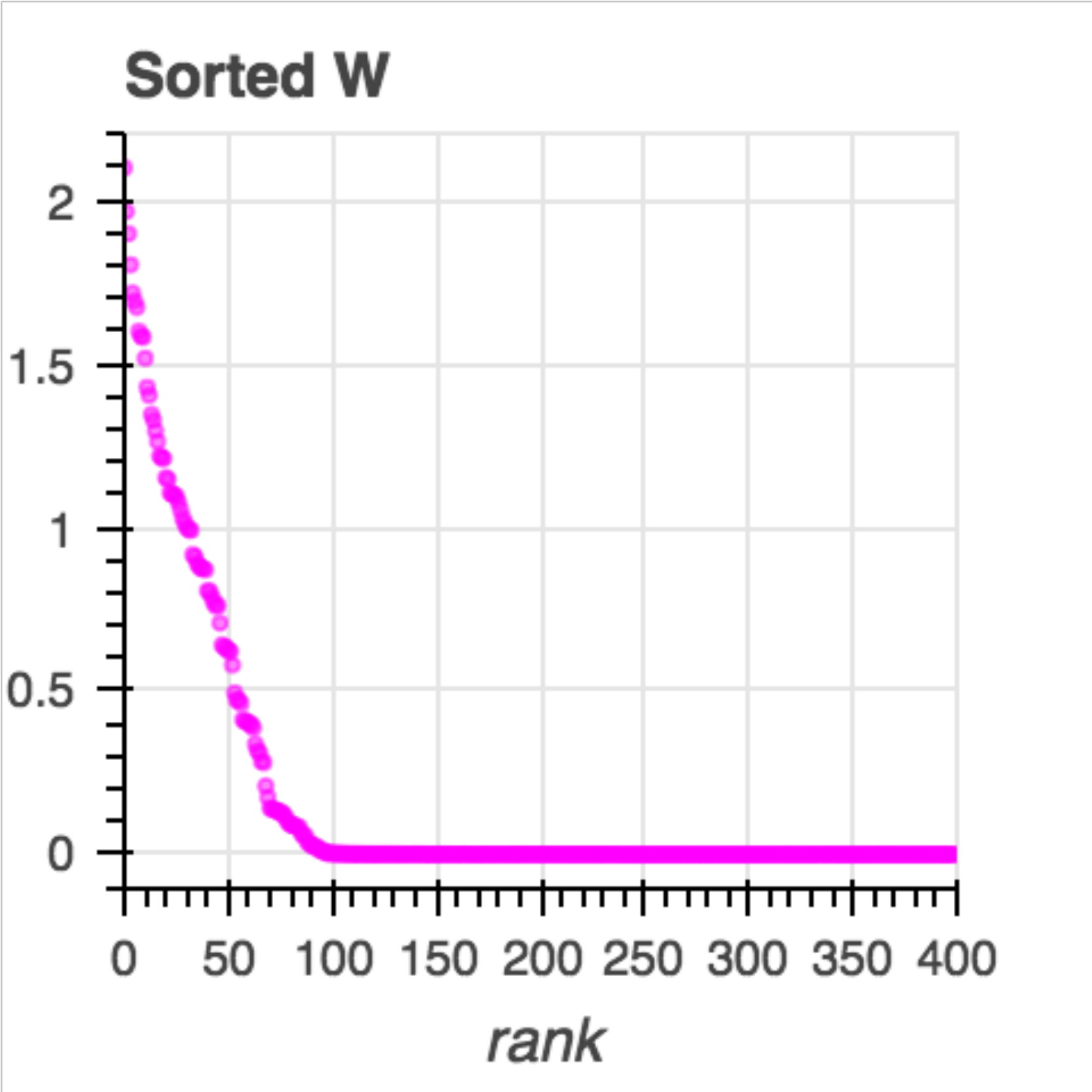}
\caption{ SF4, $ n= 1000 $ }
\end{subfigure}%
\begin{subfigure}[t]{0.49\textwidth}
\centering
\includegraphics[width=0.48\textwidth]{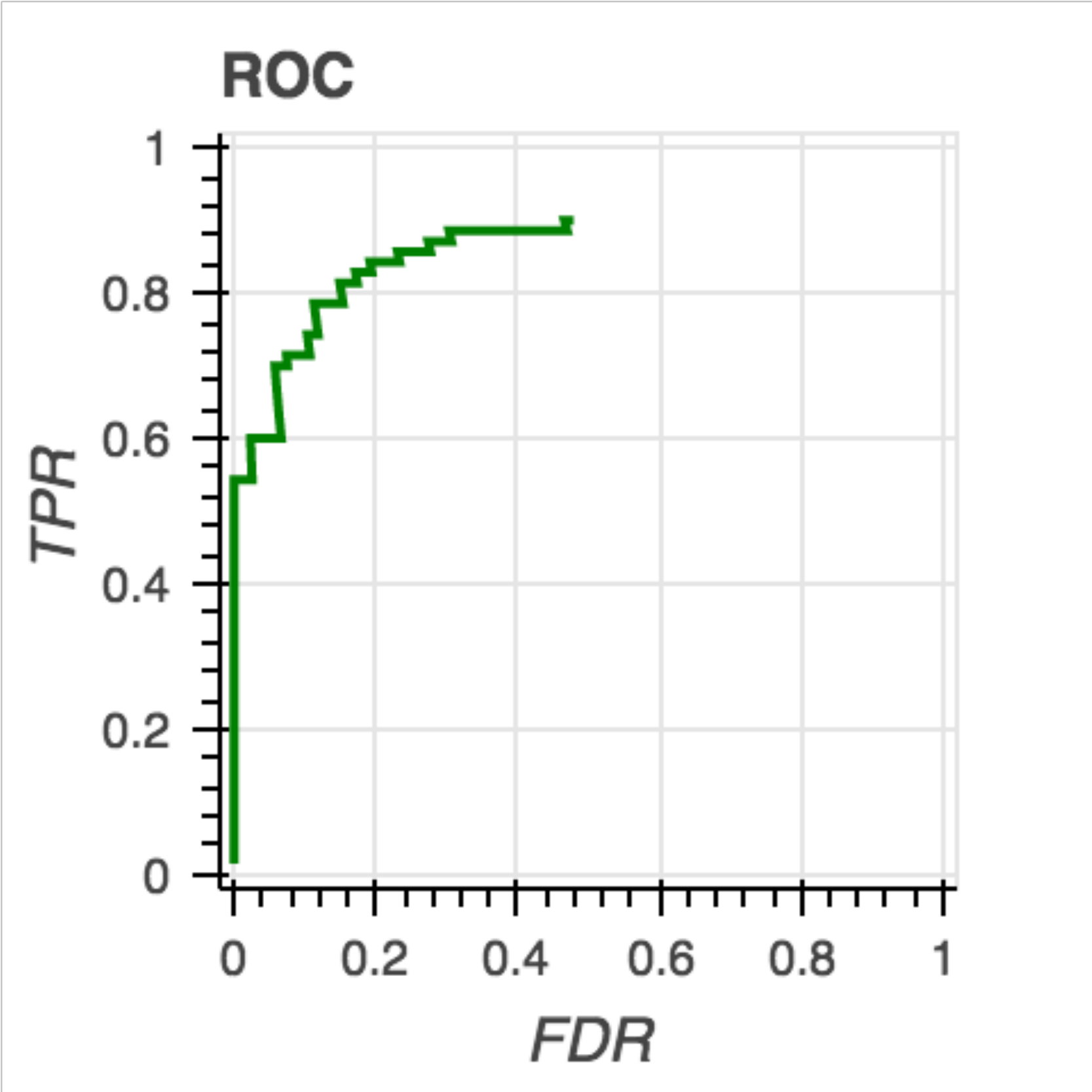}
\includegraphics[width=0.48\textwidth]{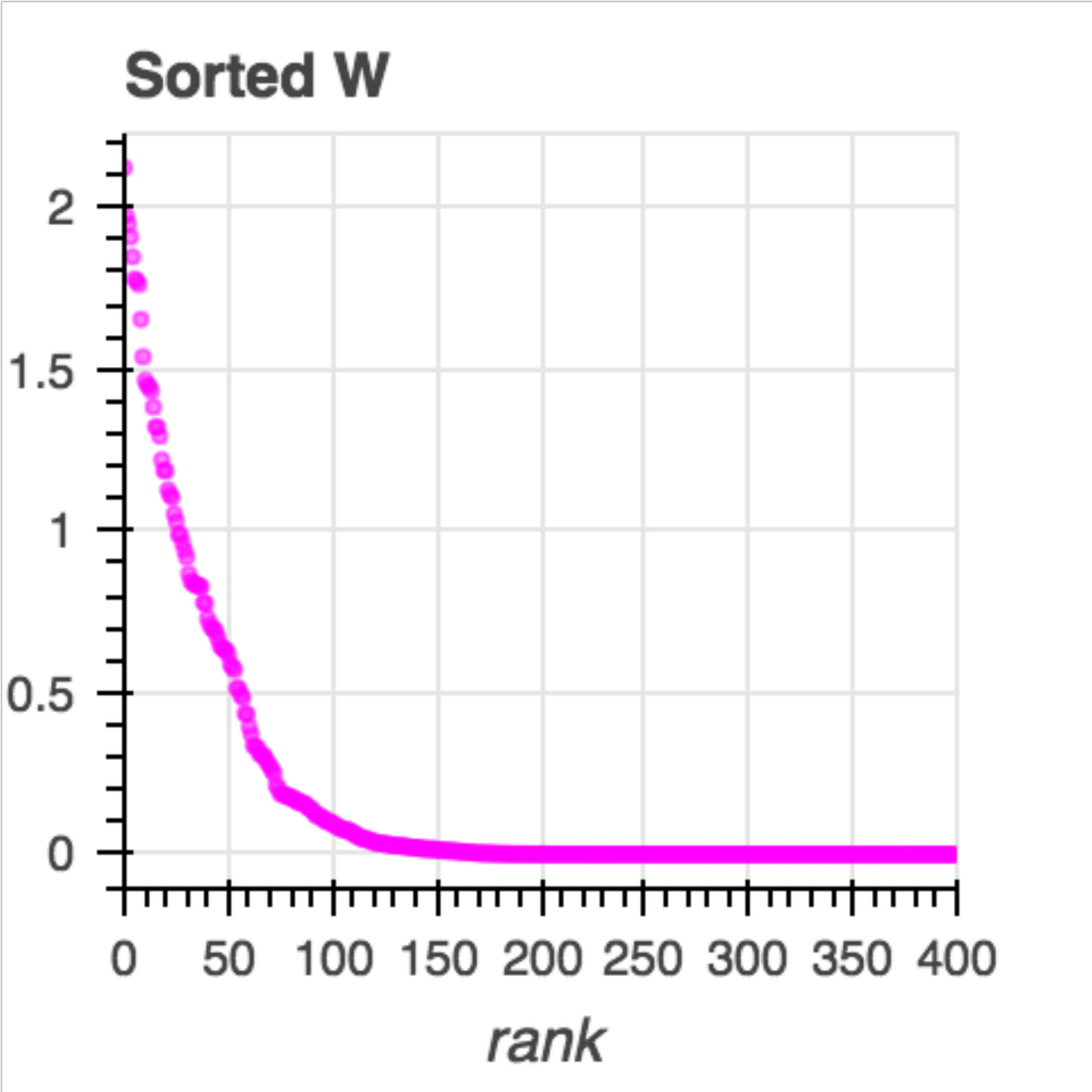}
\caption{ SF4, $ n= 20 $ }
\end{subfigure}%
\caption{Illustration of the effect of the threshold with $ d = 20 $ and $ \lambda = 0.1 $.
For each subfigure, ROC curve (left) shows FDR and TPR with varying level of threshold, and sorted weights (right) plots the entries of $ \estecp $ in decreasing order. }
\label{fig:compare:roc}
\end{figure}

\begin{figure}[t]
    \centering
    \includegraphics[width=0.48\textwidth]{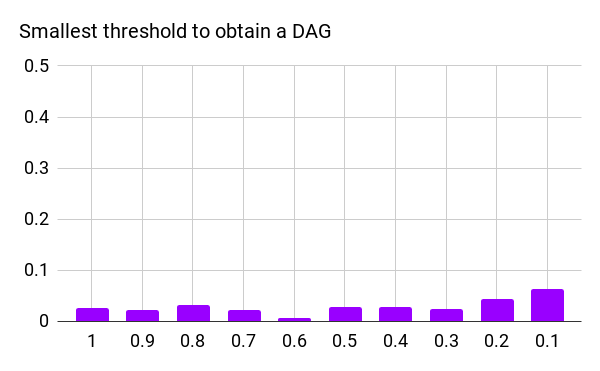}
    \includegraphics[width=0.48\textwidth]{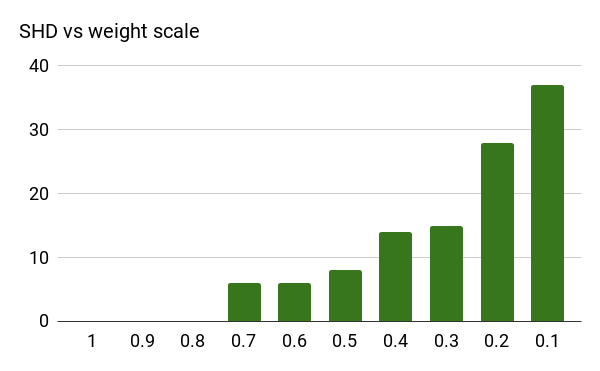}
    \caption{Varying weight scale $\alpha \in \{1.0, \dotsc, 0.1\}$ with $d=20$ and $n=1000$ on an ER-2 graph. (Left) Smallest threshold $\omega$ such that $\est$ is a DAG. (Right) SHD between ground truth and $\method$, lower the better. The minimum $\omega$ remains stable, while the accuracy of $\method$ drops as expected since the SNR decreases with $\alpha$.}
    \label{fig:scale}
\end{figure}

\section{Sensitivity of weight scale}
We investigate the effect of weight scaling to the $\method$ algorithm in Figure~\ref{fig:scale}.
In particular, we run experiments with $w_{ij}\in \alpha\cdot[0.5, 2]\cup-\alpha\cdot[0.5, 2]$ with $\alpha\in \{1.0, 0.9, 0.8, \dotsc, 0.1\}$. 
On the left, we plot the smallest threshold $\omega$ required to obtain a DAG (see Section~4.3) for different scale $\alpha$. 
Overall, across different values of $\alpha$, the variation in the smallest $\omega$ required is minimal. 
We also hasten to point out that this also decreases the signal to noise ratio (SNR), which more directly affects the accuracy. Indeed, in the figure on the right, we can observe (as expected) some performance drop when using smaller value of $\alpha$.

\section{Experiments}
\label{app:exp}

\subsection{Experiment details}
\label{app:exp:details}

We used simulated graphs from two well-known ensembles of random graphs:
\begin{itemize}
\item \emph{Erd\"os-R\'enyi (ER).} Random graphs whose edges are added independently with equal probability $p$. We simulated models with $d$, $2d$, and $4d$ edges (in expectation) each, denoted by ER-1, ER-2, and ER-4, respectively.
\item \emph{Scale-free networks (SF).} Networks simulated according to the preferential attachment process described in \citet{barabasi1999}. We simulated scale-free networks with $4d$ edges and $\beta=1$, where $\beta$ is the exponent used in the preferential attachment process.
\end{itemize}

Scale-free graphs are popular since they exhibit topological properties similar to real-world networks such as gene networks, social networks, and the internet. Given a random acyclic graph $\trueadj\in\bin$ from one of these two ensembles, we assigned edge weights independently from $\UniformDist\big([-2,-0.5]\cup[0.5,2])$ to obtain a weight matrix $\true=[\truecol_{1}\,|\,\cdots\,|\,\truecol_{d}]\in\wadj$. Given $\true$, we sampled $X=\true^{T}X+z\in\R^{d}$ according to the following three noise models:
\begin{itemize}
\item \emph{Gaussian noise} ($\Gauss$). $z\sim\normalN(0,I_{d\times d})$.
\item \emph{Exponential noise} ($\Exp$). $z_{j}\sim\ExpDist(1)$, $j=1,\ldots,d$.
\item \emph{Gumbel noise} ($\Gumbel$). $z_{j}\sim\GumbelDist(0,1)$, $j=1,\ldots,d$.
\end{itemize}

\noindent
Based on these models, we generated random datasets $\dat\in\R^{n\times d}$ by generating the rows i.i.d. according to one of the models above.
For each simulation, we generated $n$ samples for graphs with $d\in\{10,20,50,100\}$ nodes. To study both high- and low-dimensional settings, we used $n\in\{20, 1000\}$.

For each dataset, we ran FGS, PC, and LinGAM and $\method$ to compare the performance in reconstructing the DAG $B$. We used the following implementations:
\begin{itemize}
\item FGS and PC were implemented through the \texttt{py-causal} package, available at \url{https://github.com/bd2kccd/py-causal}. Both of these methods are written in highly optimized Java code.
\item LinGAM was implemented using the author's Python code: \url{https://sites.google.com/site/sshimizu06/lingam}.
\end{itemize}

Since the accuracy of PC and LiNGAM was significantly lower than either FGS or NOTEARS, we only report the results against FGS. A few comments on FGS are in order: 1) FGS estimates a graph, so it does not output any parameter estimates; 2) Instead of returning a DAG, FGS returns a CPDAG \citep{chickering2003}, which contains undirected edges; 3) FGS has a single tuning parameter that controls the strength of regularization. Thus, in our evaluations, we treated FGS favourably by treating undirected edges as true positives as long as the true graph had a directed edge in place of the undirected edge. For tuning parameters, we used the values suggested by the authors of the FGS code.

Denote the estimate returned by FGS by $\fgsest$. 
As discussed in Appendix~\ref{app:sensitivity-of-threshold}, we fix the threshold at $ \omega = 0.3 $. 
Having fixed $ \omega $, when there is no regularization, $\method$ requires no tuning. With $\ell_{1}$-regularization, $\methodreg$ requires a choice of $\lambda$ which wes selected as follows: Based on the estimate returned by FGS, we tuned $\lambda$ so that the selected graph (after thresholding) had the same number of edges as $\fgsest$ (or as close as possible). This ensures that the results are not influenced by hyperparameter tuning, and fairly compares each method on graphs of roughly the same complexity. Denote this estimate by $\est$ and the resulting adjacency matrix by $\adjest=\adj(\est)$.

\subsection{Metrics}

We evaluated the learned graphs on four common graph metrics: 1) False discovery rate (FDR), 2) True positive rate (TPR), 3) False positive rate (FPR), and 4) Structural Hamming distance (SHD). Recall that SHD is the total number of edge additions, deletions, and reversals needed to convert the estimated DAG into the true DAG.
Since we consider directed graphs, a distinction between True Positives (TP) and
Reversed edges (R) is needed: the former is estimated with correct direction
whereas the latter is not.
Likewise, a False Positive (FP) is an edge that is not in the undirected
skeleton of the true graph. 
In addition, Positive (P) is the set of estimated edges, True (T) is the set of
true edges, False (F) is the set of non-edges in the ground truth graph. 
Finally, let (E) be the extra edges from the skeleton, (M) be the missing edges
from the skeleton. The four metrics are then given by:
\begin{enumerate}[itemsep=0pt]
\item FDR $ = (\mathit{R} + \mathit{FP}) / \mathit{P} $ 
\item TPR $ = \mathit{TP} / \mathit{T} $
\item FPR $ = (\mathit{R} + \mathit{FP}) / \mathit{F} $
\item SHD $ = \mathit{E} + \mathit{M} + \mathit{R}
$.
\end{enumerate}

\subsection{Further evaluations}

Figure~\ref{fig:more-heatmap} shows learned weighted adjacency matrices for ER1 and ER4. 
One can observe the same trend: with large $ n $, both regularized and unregularized $ \method $ works well compared to FGS, and with small $ n $, due to identifiability, the unregularized $ \method $ suffers significantly, yet with the help of $ \ell_1 $-regularization we can still accurately recover the true underlying graph. 

\begin{figure}[t]
\centering
\begin{subfigure}[t]{0.138\textwidth}
\centering
\includegraphics[width=0.99\textwidth]{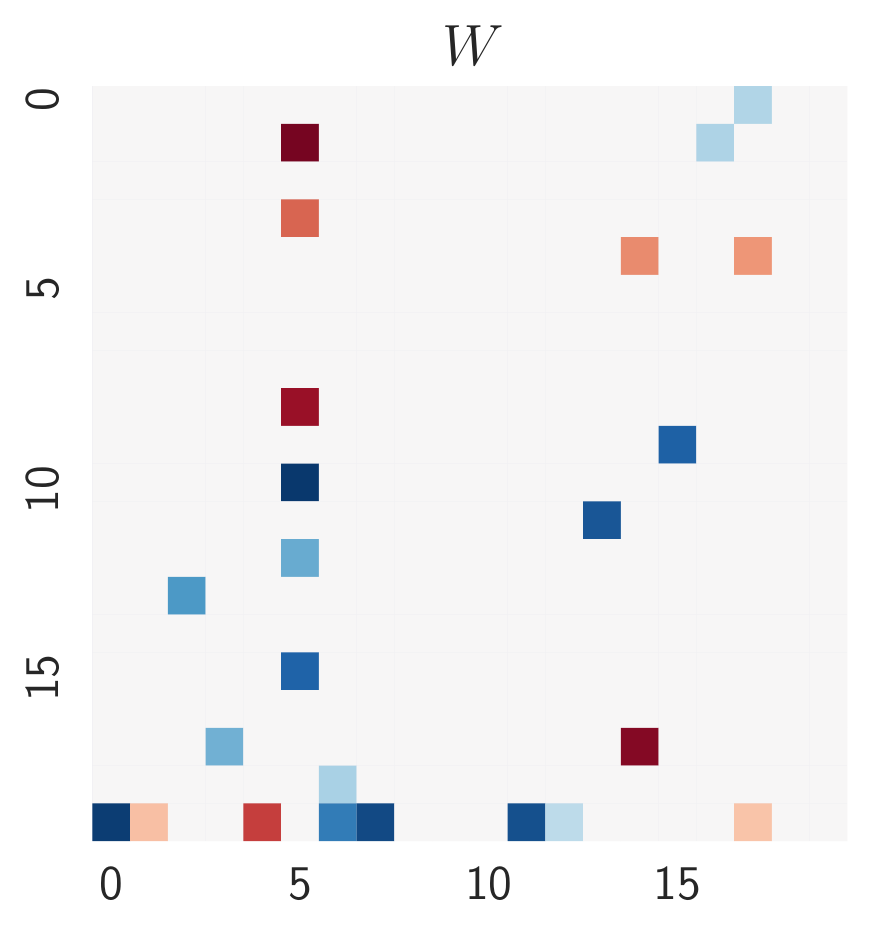}
\end{subfigure}%
~
\begin{subfigure}[t]{0.43\textwidth}
\centering
\includegraphics[width=0.99\textwidth]{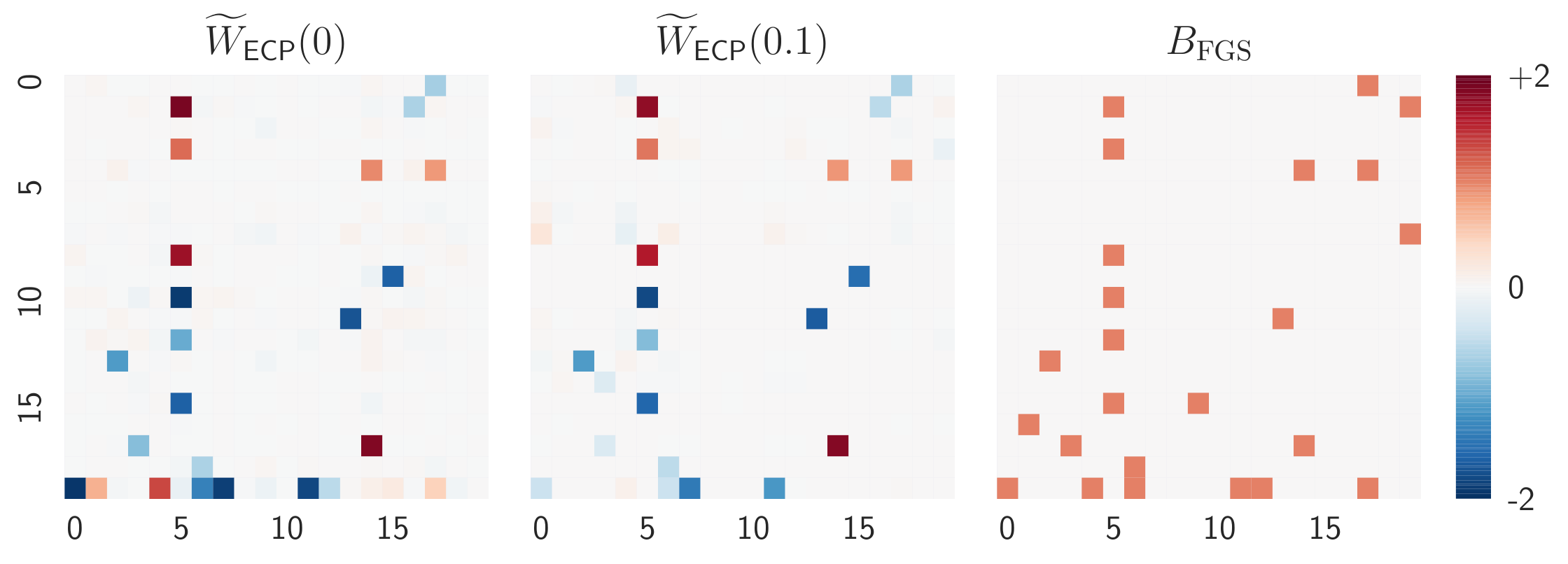}
\end{subfigure}%
\begin{subfigure}[t]{0.43\textwidth}
\centering
\includegraphics[width=0.99\textwidth]{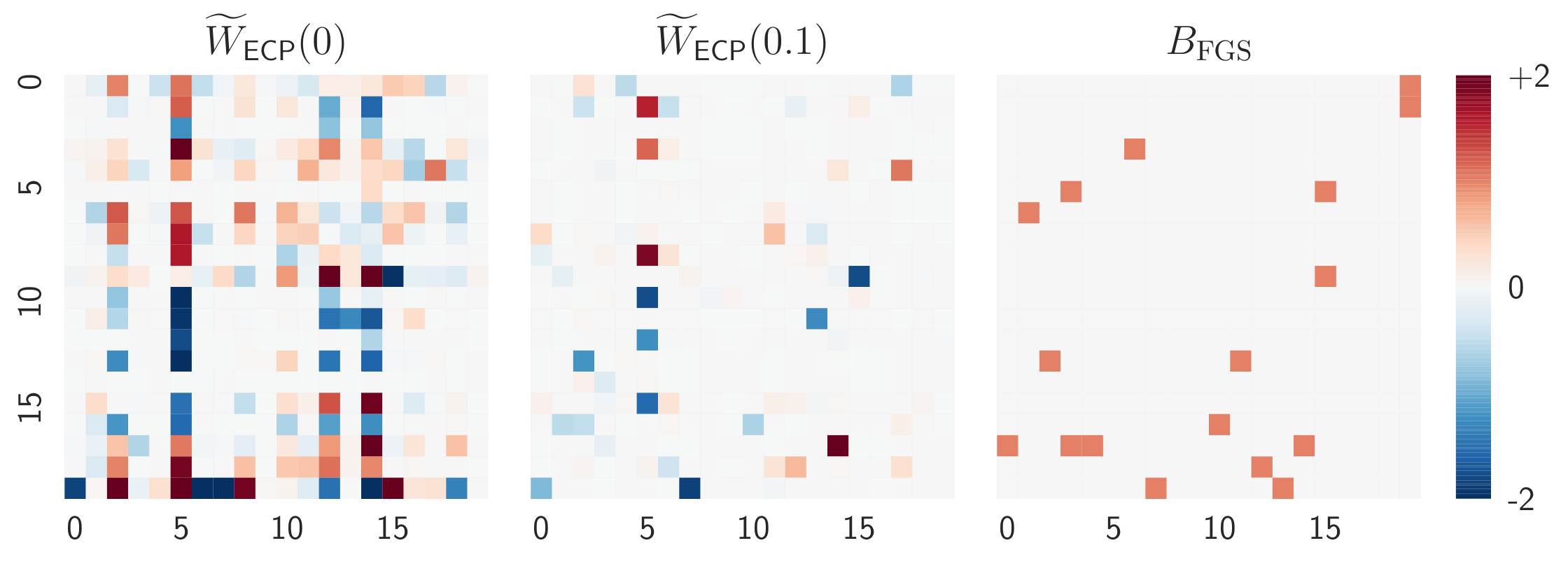}
\end{subfigure}
\begin{subfigure}[t]{0.138\textwidth}
\centering
\includegraphics[width=0.99\textwidth]{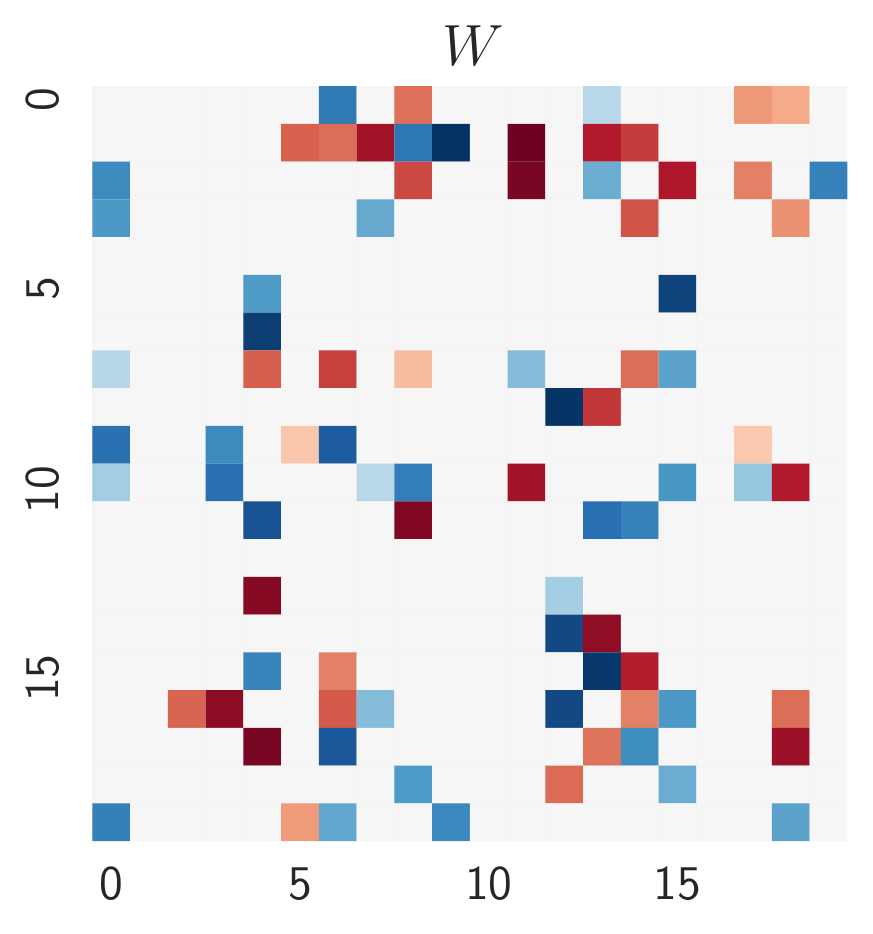}
\caption{true graph}
\end{subfigure}%
~
\begin{subfigure}[t]{0.43\textwidth}
\centering
\includegraphics[width=0.99\textwidth]{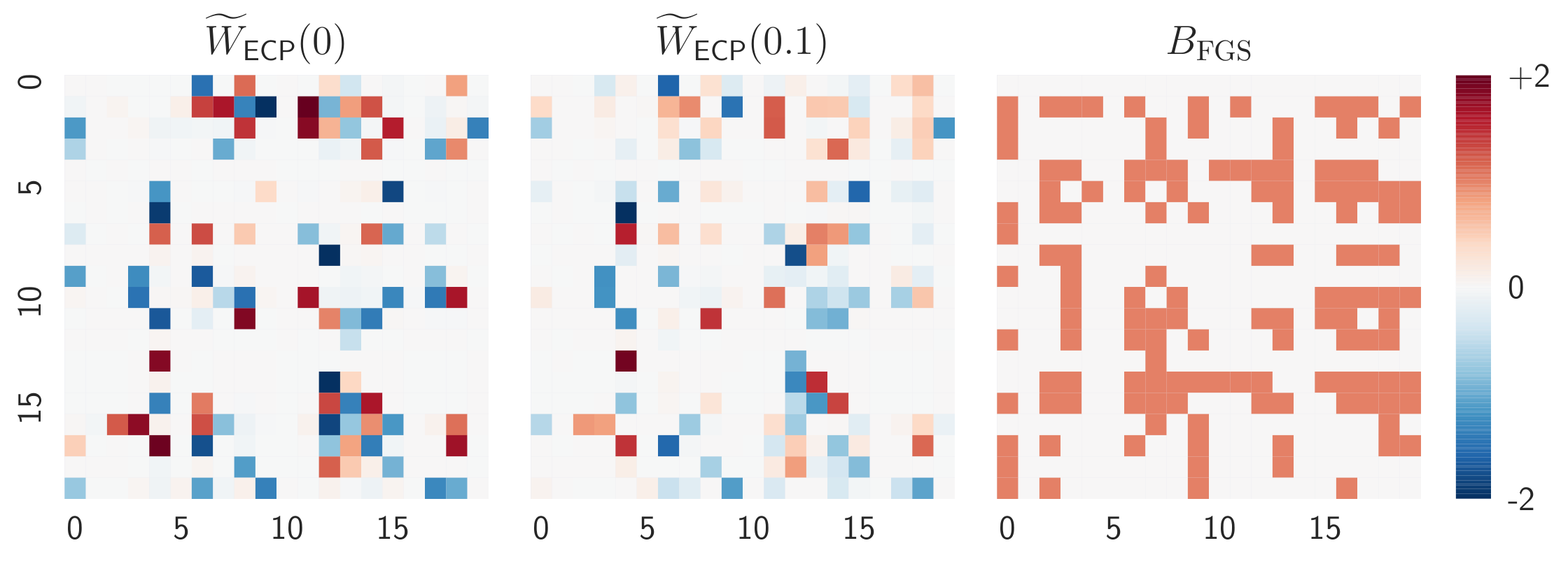}
\caption{estimate with $n=1000$}
\end{subfigure}%
\begin{subfigure}[t]{0.43\textwidth}
\centering
\includegraphics[width=0.99\textwidth]{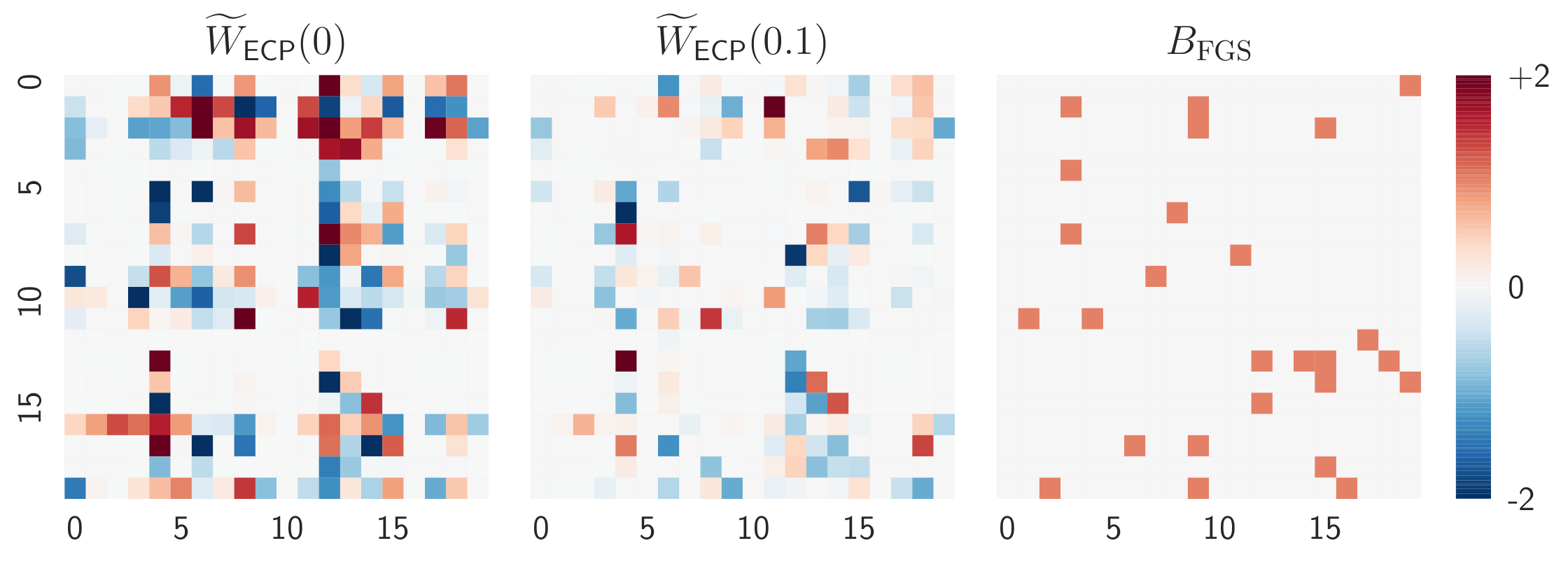}
\caption{estimate with $ n=20 $}
\end{subfigure}
\caption{Visual comparison of the learned weighted adjacency matrix on a 20-node graph with $ n=1000 $ (large samples) and $ n=20 $ (insufficient samples): $ \estecp(\lambda) $ is the proposed $ \method $ algorithm with $ \ell_1 $-regularization $ \lambda $, and $ \fgsest $ is the binary estimate of the baseline~\citep{ramsey2016}. 
Top row: ER1, bottom row: ER4.}
\label{fig:more-heatmap}
\end{figure}

Figure~\ref{fig:compare:all-n1000} and Figure~\ref{fig:compare:all-n20} shows structure recovery results for $ n=1000 $ and $ n=20 $ for various random graphs and SEM noise types. 
Other than fixed $ \omega $ as in the main paper, we also included the optimal choice of thresholding, marked as ``best''. 
The trend is consistent with the main text: our method in general outperforms FGS, without tuning $ \omega $ to the optimum for each setting.

\begin{figure}[t]
\centering
\includegraphics[width=0.99\textwidth]{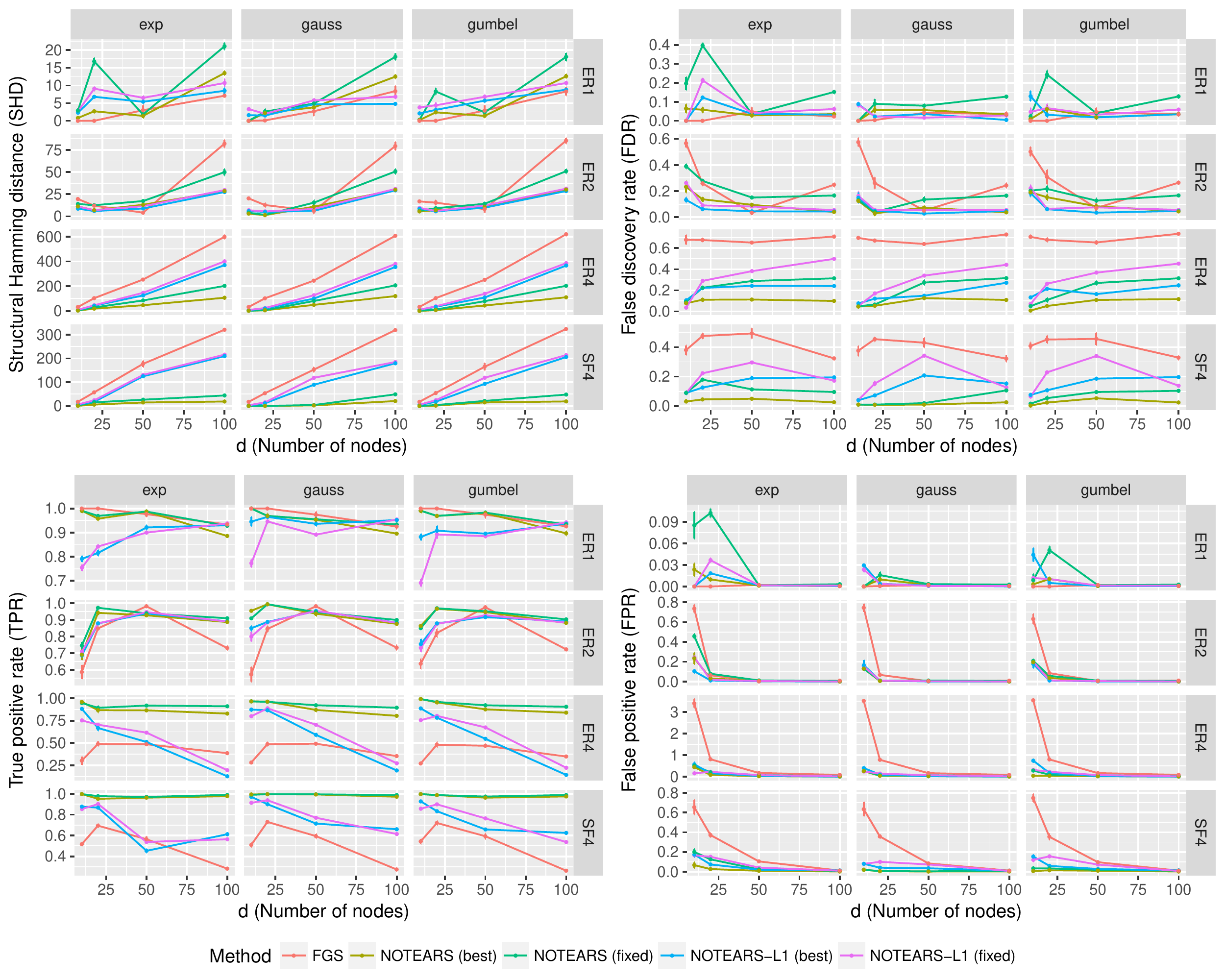}
\caption{Structure recovery results for $ n=1000 $. Lower is better, except for TPR (lower left), for which higher is better. 
Rows: random graph types, \{ER,SF\}-$ k $ = \{Erd\"os-R\'enyi, scale-free\} graphs with $ kd $ expected edges.
Columns: noise types of SEM. Error bars represent standard errors over 10 simulations.}
\label{fig:compare:all-n1000}
\end{figure}

\begin{figure}[t]
\centering
\includegraphics[width=0.99\textwidth]{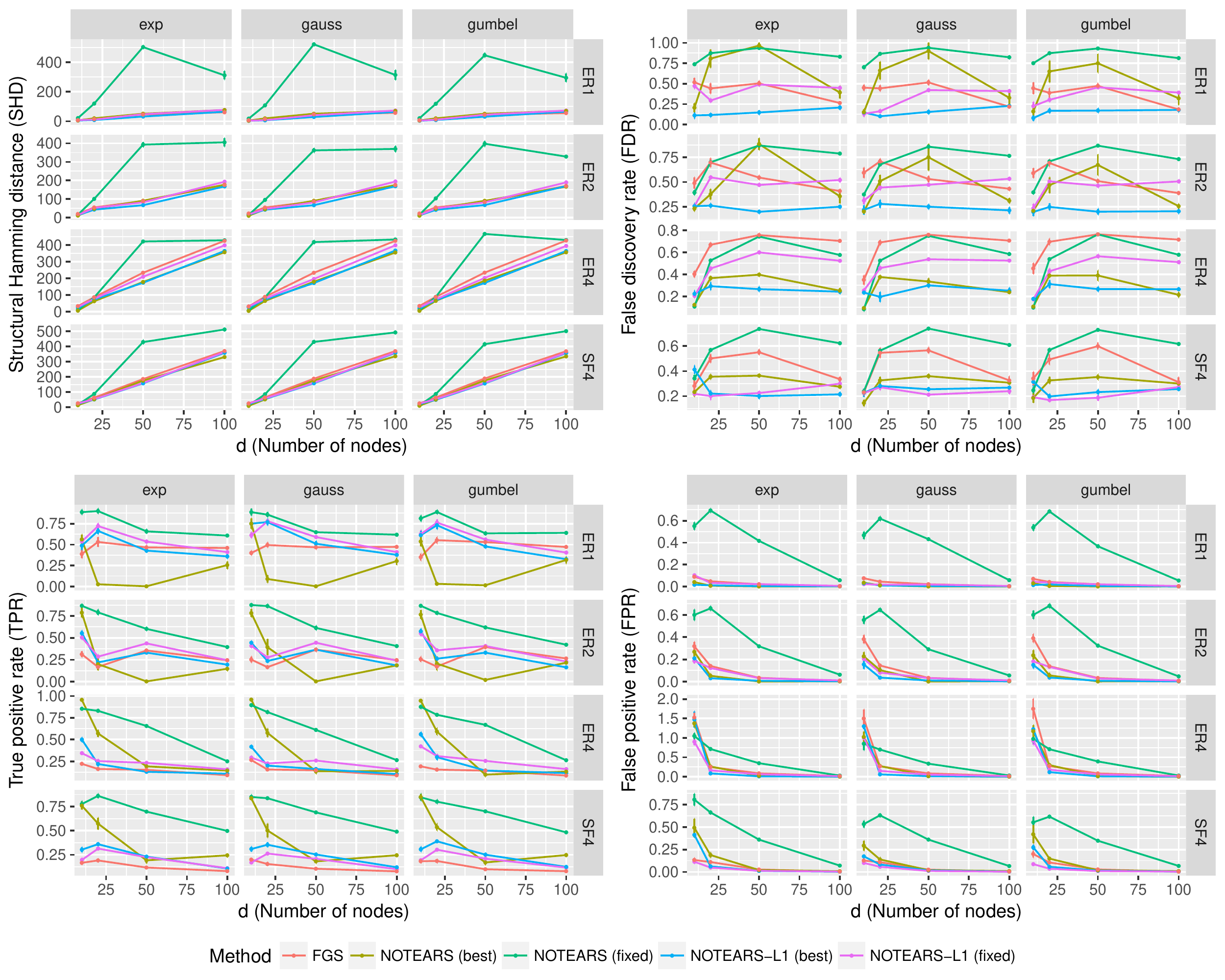}
\caption{Structure recovery results for $ n=20 $. Lower is better, except for TPR (lower left), for which higher is better. 
Rows: random graph types, \{ER,SF\}-$ k $ = \{Erd\"os-R\'enyi, scale-free\} graphs with $ kd $ expected edges.
Columns: noise types of SEM. Error bars represent standard errors over 10 simulations.}
\label{fig:compare:all-n20}
\end{figure}

Table~\ref{tab:exact:all} extends the global minimizer result for various random graph types. 
For each random graph and samples, we computed exact local scores as inputs to GOBNILP program, which finds the globally optimal structure for the given score. 
We can again observe that the difference between our estimate $ \est $ and global minimizer $ \estgob $ is small across all cases.

\begin{table}[t]
\centering
\caption{Comparison of $\method$ vs. globally optimal solution. $\Delta(\estgob,\est)=\score(\estgob) - \score(\est)$.}
\addtolength{\tabcolsep}{-2.5pt} 
\begin{tabular}{rrrrrrrrrr}
\toprule
$n$ & $\lambda$ & Graph & $\score(\true)$ & $\score(\estgob)$ & $\score(\est)$ & $\score(\estecp)$ & $\Delta(\estgob,\est)$ & $\norm{\est-W_{\gr}}$ & $\norm{\true-W_{\gr}}$ \\ 
  \midrule
    20 & 0.00 & ER1 & 5.01 & 3.69 & 5.19 & 3.73 & -1.50 & 0.09 & 3.54 \\ 
  20 & 0.50 & ER1 & 12.43 & 9.90 & 10.69 & 9.88 & -0.78 & 0.11 & 2.76 \\ 
  1000 & 0.00 & ER1 & 4.96 & 4.93 & 4.97 & 4.92 & -0.04 & 0.03 & 0.35 \\ 
  1000 & 0.50 & ER1 & 12.37 & 10.53 & 11.01 & 10.58 & -0.48 & 0.11 & 2.47 \\ 
  \midrule
20 & 0.00 & ER2 & 5.11 & 3.85 & 5.36 & 3.88 & -1.52 & 0.07 & 3.38 \\ 
  20 & 0.50 & ER2 & 16.04 & 12.81 & 13.49 & 12.90 & -0.68 & 0.12 & 3.15 \\ 
  1000 & 0.00 & ER2 & 4.99 & 4.97 & 5.02 & 4.95 & -0.05 & 0.02 & 0.40 \\ 
  1000 & 0.50 & ER2 & 15.93 & 13.32 & 14.03 & 13.46 & -0.71 & 0.12 & 2.95 \\ 
  \midrule
  20 & 0.00 & ER4 & 4.76 & 3.66 & 5.23 & 3.88 & -1.57 & 0.08 & 4.25 \\ 
  20 & 0.50 & ER4 & 28.24 & 16.38 & 19.81 & 16.82 & -3.44 & 0.15 & 6.66 \\ 
  1000 & 0.00 & ER4 & 5.03 & 5.00 & 5.50 & 4.97 & -0.50 & 0.00 & 0.46 \\ 
  1000 & 0.50 & ER4 & 28.51 & 18.29 & 29.91 & 18.69 & -11.61 & 0.13 & 5.76 \\ 
  \midrule
  20 & 0.00 & SF4 & 4.99 & 3.77 & 4.70 & 3.85 & -0.93 & 0.08 & 3.31 \\ 
  20 & 0.50 & SF4 & 23.33 & 16.19 & 17.31 & 16.69 & -1.12 & 0.15 & 5.08 \\ 
  1000 & 0.00 & SF4 & 4.96 & 4.94 & 5.05 & 4.99 & -0.11 & 0.04 & 0.29 \\ 
  1000 & 0.50 & SF4 & 23.29 & 17.56 & 19.70 & 18.43 & -2.13 & 0.13 & 4.34 \\ 
   \bottomrule
\end{tabular}
\addtolength{\tabcolsep}{2.5pt}
\label{tab:exact:all}
\end{table}

\end{document}